\documentclass{article}

    \PassOptionsToPackage{numbers, sort}{natbib}
    \usepackage[final]{neurips_2024}

\usepackage[utf8]{inputenc} %
\usepackage[T1]{fontenc}    %
\usepackage{hyperref}       %
\hypersetup{
  colorlinks,
  linkcolor={red!50!black},
  citecolor={blue!50!black},
  urlcolor={blue!50!black}
}
\usepackage{url}            %
\usepackage{booktabs}       %
\usepackage{amsfonts}       %
\usepackage{amsthm}
\usepackage{mathtools}
\usepackage{nicefrac}       %
\usepackage{microtype}      %
\usepackage{xcolor}         %
\usepackage{amsmath}
\usepackage{amssymb}
\usepackage{tablefootnote}
\usepackage{graphicx}
\usepackage{multirow}
\usepackage{wrapfig}
\usepackage{enumitem}
\usepackage{listings}
\usepackage{subfig}
\usepackage[export]{adjustbox}
\usepackage{algorithm}
\usepackage{algorithmicx}
\usepackage[noend]{algpseudocode}
\usepackage{cleveref}

\newcommand{\numtokens}
{N}
\newcommand{\indextoken}
{n}

\newcommand{\E}{\mathbb{E}}
\newcommand{\diag}{\mathrm{diag}}
\newcommand{\KL}[2]{\mathrm{KL}({#1}\|{#2})}

\def \rmd {\mathrm{d}}
\def \logsnr {\textup{log-SNR}}
\Crefname{equation}{Eq.}{Eqs.}
\Crefname{figure}{Fig.}{Figs.}
\Crefname{table}{Tab.}{Tabs.}
\Crefname{appendix}{App.}{Apps.}
\Crefname{section}{Sec.}{Secs.}
\Crefname{algorithm}{Alg.}{Algs.}

\ifdefined\nonewproofenvironments\else
\ifdefined\ispres\else
\newtheorem{theorem}{Theorem}
\newtheorem{lemma}[theorem]{Lemma}
\newtheorem{proposition}[theorem]{Proposition}

\theoremstyle{definition}

\renewenvironment{proof}{\noindent\textbf{Proof}\hspace*{1em}}{\qed\\}

\newenvironment{proofof}[1]{\noindent\textbf{Proof of {#1}}\ }{\hfill\qed\\}

\newenvironment{talign*}
 {\csname align*\endcsname}
 {\endalign}
\newenvironment{talign}
 {\csname align\endcsname}
 {\endalign}

\def\textsum{{\textstyle\sum}} %

\title{Simplified and Generalized\\
Masked Diffusion for Discrete Data}

\author{%
  Jiaxin Shi\thanks{Equal contribution. Correspondence to: \texttt{jiaxins@google.com}.} , Kehang Han\footnotemark[1] , Zhe Wang, Arnaud Doucet, Michalis K. Titsias \\
  Google DeepMind \\
}

\begin{document}

\maketitle

\begin{abstract}
Masked (or absorbing) diffusion is actively explored as an alternative to autoregressive models for generative modeling of discrete data. 
However, existing work in this area has been hindered by unnecessarily complex model formulations and unclear relationships between different perspectives, leading to suboptimal parameterization, training objectives, and ad hoc adjustments to counteract these issues.
In this work, we aim to provide a simple and general framework that unlocks the full potential of masked diffusion models. 
We show that the continuous-time variational objective of masked diffusion models is a simple weighted integral of cross-entropy losses. Our framework also enables training generalized masked diffusion models with state-dependent masking schedules. When evaluated by perplexity, our models trained on OpenWebText surpass prior diffusion language models at GPT-2 scale and demonstrate superior performance on 4 out of 5 zero-shot language modeling tasks.
Furthermore, our models vastly outperform previous discrete diffusion models on pixel-level image modeling, achieving 2.75~(CIFAR-10) and 3.40 (ImageNet 64$\times$64) bits per dimension that are better than autoregressive models of similar sizes.
Our code is available at \url{https://github.com/google-deepmind/md4}.
\end{abstract}

\section{Introduction}
Since their inception \citep{sohl2015deep,ho2020denoising,song2020score}, diffusion models have emerged as the workhorse for generative media, achieving state-of-the-art in tasks such as image synthesis~\citep{rombach2022high,ramesh2022hierarchical,saharia2022photorealistic}, audio~\citep{chen2021wavegrad,kong2021diffwave} and video generation~\citep{ho2022video,villegas2023phenaki,bar2024lumiere,openai2024sora,bao2024vidu}. 
The majority of existing successes are for continuous state space diffusions. While diffusion models have been extended to discrete state spaces \citep{sohl2015deep,austin2021structured,hoogeboom2021argmax} and have been successfully applied to applications ranging from graph generation \citep{vignac2023digress}, text-to-sound generation \citep{yang2023diffsound} or protein design \citep{gruver2024protein}, they remain not as widely used as their continuous counterparts as they are not competitive with autoregressive models in important domains such as text modeling. This has motivated the development of continuous space diffusion models where the discrete data are embedded in the Euclidean space \citep{dieleman2022continuous,chen2022analog,li2022diffusion,gulrajani2024likelihood,lovelace2024latent} or the simplex~\citep{richemond2022categorical,avdeyev2023dirichlet,graves2023bayesian,xue2024unifying,liu2024mirror}. We believe that one of the reasons for the limited success of discrete diffusions is that they have been hindered by fairly complex formulations and training objectives. This paper is a step towards closing this gap.

In this work, we focus on ``masked'' (or ``absorbing'') diffusions, a discrete diffusion formulation first presented by \citet{austin2021structured}, and later explored by the literature from various perspectives
\citep{campbell2022continuous,sunscore,zheng2023reparameterized,lou2023discrete}. We follow here a continuous-time framework which has proven very useful to improve the training and understanding of continuous state space diffusions \citep[see e.g.,][]{song2020score,kingma2021variational,karras2022elucidating}. We make several technical contributions which simplify the training of these models and improve significantly their performance. Our contributions are as follows:
\begin{itemize}[wide=5pt]
    \item Using elementary arguments, we establish several properties for the forward process induced by this model and its corresponding time reversal,
    improving our understanding of this model class.
    \item We provide a remarkably simple expression of the Evidence Lower Bound (ELBO) for masked diffusion models, showing that it corresponds to a weighted integral over time of cross-entropy losses. Similarly to continuous space diffusions \citep{kingma2021variational}, this objective can be rewritten in terms of signal-to-noise ratio and exhibits invariance properties.
    \item We develop a unifying understanding of previously proposed continuous-time discrete diffusion models~\citep{campbell2022continuous,benton2024denoising,lou2023discrete}, revealing the changes they made to our ELBO objective and/or model parameterization. 
    We show that these changes either lead to expensive model evaluations, or large variance in training, or breaking the consistency between forward and reverse processes. 
    \item On GPT-2 scale text modeling and pixel-level image modeling tasks, masked diffusions trained using our simple ELBO objective outperform previous proposals, leading to the best likelihood and zero-shot transfer performance among discrete diffusion models.
    \item Finally, based on our simplified masked diffusion formulation, we propose a generalized masked diffusion model that allows state-dependent masking schedules. This generalized masked diffusion model further improves predictive performance measured by test likelihoods. 
\end{itemize}

Concurrent work by \citet{ou2024} and \citet{sahoo2024simple} derives a similar simplified expression of the ELBO. \citet{ou2024}'s derivation relies on an observation similar to the one we made in \Cref{prop:score-mean-relation}.

\section{Masked Diffusion
\label{sec:masked_diffusion}
}
Consider a sentence where we progressively replace each word with a special mask token, transforming the sentence into a sequence of masks.
Our goal is to train a generative model that reverses this process, effectively turning a sentence of masks back into meaningful text.
More formally, assume our data consists of tokens from a finite discrete state space with $m$ possible states, represented by integers
$0, 1, \dots, m - 1$ and
their corresponding one-hot vectors $e_0, e_1, \dots, e_{m - 1}$. To accommodate the masking process, we augment this space with an additional mask state, denoted by the index $m$.
The masking process transitions each token to
the mask state at a random time.
This process, known as the forward process, is applied independently to each token (e.g., each word), progressively converting the data into a sequence of mask tokens. 
By learning to reverse this masking process, we create a generative model capable of producing coherent discrete data. 

\paragraph{Discrete-time forward process.} 
We start with the case of a single token and later expand to multiple dimensions.
We define the forward process as a Markovian sequence of discrete random variables $x_t$ indexed by time $t$, where $t$ runs from 0 to 1. 
Throughout the work, we abuse the notation such that $x_t$ can be either an integer or its corresponding one-hot vector, whenever it is clear from the context. 
We divide $[0, 1]$ into $T$ intervals, and let $s(i) = (i - 1)/T$, $t(i) = i / T$. 
Following \citet{austin2021structured}, the state transition between $[s(i), t(i)]$ is determined by a transition matrix of size $(m + 1) \times (m + 1)$:
    $Q_i = (1 - \beta_i)I + \beta_i \mathbf{1} e_m^\top,$
where $\mathbf{1}$ is an all-one vector of size $m + 1$, $e_m$ represents a one-hot vector where element at index $m$ is 1. 
Each entry $[Q_i]_{jk}$ denotes the probability of transition from the state $j$ to the state $k$:
\begin{align*}
    [Q_i]_{jk} = q(x_{t(i)} = k|x_{s(i)} = j) = (1 - 
    \beta_i)\delta_{jk} + \beta_i \delta_{km}.
\end{align*}
This means that, with probability $1 - \beta_i$, $x_{t(i)}=x_{s(i)}$, otherwise it jumps to the mask state. 
Given the above transition matrix, the marginal distribution at time $t(i)$ given $x_0$ is 
\begin{align*}
    q(x_{t(i)}|x_0) = \mathrm{Cat}(x_{t(i)}; \bar{Q}_i^\top x_0) = x_0^\top \bar{Q}_i x_{t(i)}.
\end{align*}
Here, we use $\mathrm{Cat}(x;p)$ to denote a Categorical distribution where $p$ is the vector of probabilities of being in each category, and $\bar{Q}_i \triangleq \prod_{j=1}^i Q_j
= \alpha_i I + \big(1 - \alpha_i \big)\mathbf{1}e_m^\top$ for $\alpha_i = \prod_{j=1}^i (1 - \beta_j)$. 
We expect $\alpha_T$ to become very small or zero for a sufficiently large $T$ such that $q(x_1|x_0)$ for any $x_0$ will become a delta mass at the mask state. 

\paragraph{Continuous-time limit.} 
We can define a continuous-time forward process by taking a limit of the above discrete-time process. We first specify a continuous function $\beta(t)$ such that 
$\beta_i = \beta(t(i))/T$. 
We then let $T \to \infty$ in the discrete-time process and compute the limit of $\bar{Q}_i$ (proved in \citealt[Appendix A.6]{austin2021structured}, see also \Cref{app:discrete-time-derivation}%
) as
\begin{align}
    \bar{Q}(t) \triangleq \lim_{T\to \infty} \bar{Q}_i 
    &= \alpha_t I + (1 - \alpha_t)\mathbf{1}e_m^\top, \text{ where } \alpha_t \triangleq \exp\Big(-\int_0^{t} \beta(s) \rmd s\Big),  \label{eq:bar-Qt-from-discrete} 
\end{align}
so that $q(x_t|x_0) = \mathrm{Cat}(x_t; \bar{Q}(t)^\top x_0)$. 
For two arbitrary times, $0 \leq s < t \leq 1$, the transition distribution that is compatible with the above marginal %
(i.e., $q(x_t|x_0) = \sum_{x_s} q(x_t|x_s)q(x_s|x_0)$) is 
\begin{align*} %
    q(x_t|x_s) &= \mathrm{Cat}(x_t; \bar{Q}(s, t)^\top x_s), 
    \text{ where } \bar{Q}(s, t) \triangleq \bar{Q}(s)^{-1}\bar{Q}(t) = \frac{\alpha_t}{\alpha_s} I + \big(1 - \frac{\alpha_t}{\alpha_s}\big)\mathbf{1}e_m^\top. 
\end{align*}
Note that \citet{austin2021structured} did not derive this explicit form of transition matrix between two arbitrary time $s$ and $t$, which appeared later in \citet{zhao2024improving} concurrently with our work.

\begin{figure}[t]
    \centering
    \includegraphics[width=0.94\textwidth]{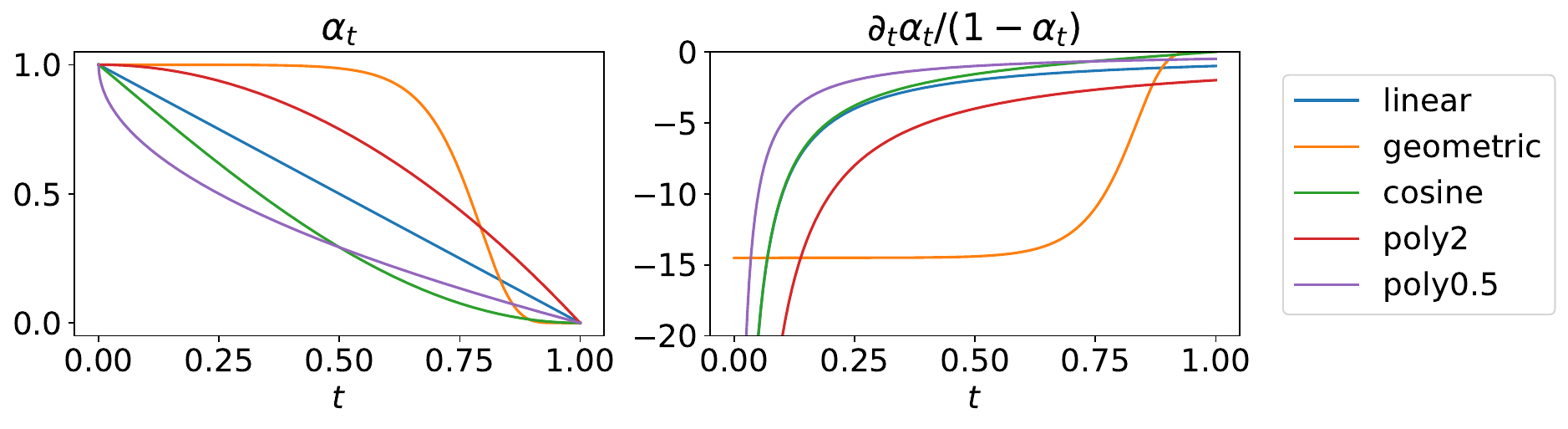}
    \caption{Masking schedules in the literature: (Left) $\alpha_t$; (Right) weight of the cross-entropy loss w.r.t. $t$; Equations for these schedules are given in \Cref{tab:noise-schedules} in Appendix. }
    \label{fig:noise-schedules}
\end{figure}

\paragraph{Masking schedules. } 
From the definition of $\alpha_t$, we have that $\alpha_0 = 1$. 
And similar to the discrete-time formulation, we would like $\alpha_1$ be zero or very close to zero. 
We provide a summary of masking schedules from literature that satisfy these properties in \Cref{fig:noise-schedules}. 
The linear schedule was proposed in \citet{sohl2015deep} for binary variables and then re-derived by \citet{austin2021structured} from mutual information for discrete-time models. The geometric schedule $\alpha_t$ 
is plotted for $\bar{\beta}_{\text{min}} = 10^{-5}$ and $\bar{\beta}_{\text{max}} = 20$. It was first used for continuous diffusions \citep{song2020score} and then for discrete %
by \citet{lou2023discrete}. 
The cosine schedule was originally proposed in MaskGIT~\citep{chang2022maskgit}, an iterative unmasking generative model inspired by diffusion. 
This schedule has the property of slowing down the unmasking process at the beginning of the reverse generation. 
Aligning with their observation, we find that this results in a lower chance of conflicting tokens being unmasked simultaneously at the start of generation, thereby enhancing the overall generation quality.

\paragraph{Time reversal of the forward process given $x_0$.}

The analytic property of our forward process allows to compute many quantities of interest in closed form. 
One such quantity frequently used in diffusion models is the time reversal of the forward process given $x_0$: $q(x_s|x_t, x_0)$ for $s\leq t$. 
We %
derive it in \Cref{app:time-reversal-x0} as
\begin{align*} %
    q(x_s|x_t, x_0) = \mathrm{Cat}(x_s; \bar{R}^{x_0}(t, s)^\top x_t), \text{ where } \bar{R}^{x_0}(t, s) = I + \frac{\alpha_s - \alpha_t}{1 - \alpha_t}e_m(x_0 - e_m)^\top.
\end{align*}
From the transition matrix $\bar{R}^{x_0}(t, s)\in \mathbb{R}^{(m+1)\times (m+1)}$ 
we can see the reverse process conditioned on $x_0$ has a very simple logic---if $x_t$ is a mask, with probability $\frac{\alpha_s - \alpha_t}{1 - \alpha_t}$, it will jump to the state $x_0$ at time $s$, otherwise it will stay masked. Once $x_t$ is unmasked, it remains in the same state until the end.

\section{Model and Objective
\label{sec:model_objective}
}

For a discrete-time masked diffusion process, 
we define our generative model by approximately reversing the forward transitions using a reverse model $p_\theta(x_s|x_t)$.  
One %
way to define this %
model is 
\begin{align} \label{eq:backward-model}
    p_\theta(x_s|x_t) \triangleq q(x_s|x_t, \mu_\theta(x_t, t)),
\end{align}
where $\mu_\theta(x_t, t) \in \Delta^{m+1}$ is a probability vector parametrized by a neural network $f_\theta$ with a softmax applied to the output logits (note the $m$-th output is %
forced to 0 since the clean data cannot be masks):
\begin{align} \label{eq:hat-x0}
    \mu_\theta(x_t, t) = \begin{cases} 
    \mathrm{softmax}(f_\theta(x_t, t)) & x_t = m, \\
    x_t & x_t \neq m.
    \end{cases}
\end{align}
This is known as mean-parameterization since it leverages a prediction model for the mean of $x_0$. A matrix-form depiction of $p_\theta(x_s | x_t)$ is shown in \Cref{fig:reverse-transition-matrix} (right). 
In fact, we can select a time-invariant parametrization $\mu_\theta(x_t, t)=\mu_\theta(x_t)$ as \cite{ou2024} showed that $p(x_0|x_t)$ given $x_t = x$ is identical for any $t$.  

Besides $p_\theta(x_s|x_t)$, we also need to specify $p(x_0|x_{t(1)})$ and the prior distribution $p(x_{t(T)})=p(x_1)$. 
Following the practice in continuous diffusion models~\citep{kingma2021variational}, we choose $p(x_0|x_{t(1)}) \propto q(x_{t(1)}|x_0)$. 
And since $q(x_1|x_0) \approx %
\delta_{x_1, m}$ for any $x_0$ as $\alpha_1\approx 0$, we set $p(x_1) \approx %
\delta_{x_1, m}$, see \Cref{app:undefined-kl}. 

We then write out the discrete-time diffusion model objective~\citep{sohl2015deep,ho2020denoising}, which is a lower bound of the log marginal likelihood of data $x_0$ under the model $p$ (known as the Evidence Lower Bound, or ELBO):
\begin{align*}
    \log p(x_0) &\geq 
    \E_{q(x_{t(1)}|x_0)} [\log p(x_0|x_{t(1)})] - \KL{q(x_1|x_0)}{p(x_1)} 
    - \mathcal{L}_T,
\end{align*}
where $\mathcal{L}_T = \sum_{i=2}^T \E_{q(x_{t(i)}|x_0)} [\KL{q(x_{s(i)}|x_{t(i)}, x_0)}{p_\theta(x_{s(i)}|x_{t(i)})}]$.
For the above choices of the prior distribution, the term $\KL{q(x_1|x_0)}{p(x_1)}$ becomes zero. 
Under the reverse model \eqref{eq:backward-model}, the KL divergence terms in $\mathcal{L}_T$ becomes (proof in \Cref{app:elbo})
\begin{align*}
    \KL{q(x_s|x_t, x_0)}{p_\theta(x_s|x_t)} 
    = -\frac{\alpha_s - \alpha_t}{1 - \alpha_t}\delta_{x_t, m} \cdot x_0^\top \log \mu_\theta(x_t, t), %
\end{align*}
which is a simple cross-entropy loss between the predicted logits and the clean data. 
In \Cref{app:elbo}, we show that $\mathcal{L}_T$ is a Riemann sum and is lower bounded by the corresponding continuous integral:
\begin{align} 
    \mathcal{L}_\infty &\triangleq \lim_{T\to \infty} \mathcal{L}_T 
    = \int_{t(1)}^1 \frac{\alpha_t'}{1 - \alpha_{t}} \E_{q(x_{t}|x_0)}\left[\delta_{x_t, m}\cdot x_0^\top \log \mu_\theta(x_{t}, t)\right] \rmd t,
    \label{eq:diff-kl-limit}
\end{align}
where $\alpha'_t$ denotes the derivative of $\alpha_t$ with respect to $t$.
Therefore, we can obtain an ELBO that is tighter than that of any finite $T$ by pushing $T\to \infty$. 
This ELBO can be further simplified by letting $t(1) \to 0$. 
As a result, $\E_{q(x_{t(1)}|x_0)}[\log p(x_0|x_{t(1)})]$ goes to $0$ and the ELBO becomes $-\mathcal{L}_\infty$.

For continuous state-space diffusions, the ELBO depends on the signal-to-noise ratio (SNR) at its endpoints but is otherwise invariant to the noise schedule \citep{kingma2021variational}. We establish here a similar result for discrete diffusions. 
Consider choosing $\alpha_t = \sigma(\lambda_t)$, where $\sigma$ represents the sigmoid function $\sigma(x) = \frac{1}{1 + e^{-x}}$. 
In this context, the \logsnr~is defined by $\lambda_t = \log \frac{\alpha_t}{1 - \alpha_t}=\logsnr(t)$. 
By making a change of variables in \eqref{eq:diff-kl-limit} to make everything a function of the \logsnr, we obtain 
\begin{equation}
    \mathcal{L}_{\infty} 
    = \int_{\lambda_{t(1)}}^{\lambda_1} \sigma(\lambda)  \E_{\tilde{q}(x_{\lambda}|x_0)}\left[\delta_{x_\lambda,m}
    \cdot x_0^\top\log \tilde{\mu}_\theta(x_{\lambda},\lambda)\right] \rmd \lambda. 
    \nonumber
\end{equation}
where $\tilde{\mu}_\theta(x,\lambda):=\mu_\theta(x,t)$ and $\tilde{q}(x_\lambda|x_0):=q(x_t|x_0)$ for $t=\logsnr^{-1}(\lambda)$. This shows that the only effect $\alpha_t$ has on the loss is through the values of the SNR at the endpoints.
Still, because we draw uniform samples of $t$ to estimate the %
integral, the choice of masking schedule affects the %
variance. 
\paragraph{Multidimensional data.}
In the previous sections, $x_t$ was assumed to be a single discrete token. 
To extend the method to multidimensional data, let $x_t$ be now %
a sequence $(x_t^{(1)}, x_t^{(2)}, \dots, x_t^{(\numtokens)})$,
where each %
element $x_t^{(\indextoken)}$ 
represents a discrete token. We select a forward process which factorizes across all 
 $\numtokens$ tokens: $q(x_t|x_s) = \prod_{\indextoken=1}^\numtokens q(x_t^{(\indextoken)}|x_s^{(\indextoken)})$. 
As a result, the forward marginals $q(x_t|x_0)$ and reversal $q(x_s|x_t, x_0)$ also factorize. 
In this case, we define the reverse model as 
$p_\theta(x_s|x_t) \triangleq \prod_{\indextoken=1}^\numtokens q(x_s^{(\indextoken)}|x_t^{(\indextoken)}, \mu_\theta^{(\indextoken)}(x_t, t))$, where $\mu_\theta(x_t, t)$ is a neural network that takes the full $\numtokens$ tokens as input and 
outputs $\numtokens$ probability vectors.\footnote{We intentionally choose the reverse model to factorize across dimensions because the true reverse transition $q(x_s|x_t)$ factorizes in the continuous-time limit (as $s$ approaches $t$).}
The $\indextoken$-th output $\mu_\theta^{(\indextoken)}(x_t, t)$ is a prediction model for $\mathbb{E}[x_0^{(\indextoken)}|x_t]$, the mean value of the $\indextoken$-th token.
Repeating above derivations gives 
\begin{align}
    \mathcal{L}_\infty^{(\numtokens)} &\triangleq \int_0^1 \frac{\alpha_t'}{1 - \alpha_{t}} \E_{q(x_{t}|x_0)}\Big[\textsum_{\indextoken: x_t^{(\indextoken)}=m} (x_0^{(\indextoken)})^\top \log \mu_\theta^{(\indextoken)}(x_{t}, t)\Big] \rmd t. 
    \label{eq:multi-diff-kl-limit}
\end{align} 
We term our simple masked diffusion model trained with loss \eqref{eq:multi-diff-kl-limit} \textbf{MD4} (Masked Discrete Diffusion for Discrete Data).
A single step of MD4 training algorithm is described in \Cref{alg:training} in Appendix.

\section{Sampling}
\label{sec:sampling}

We use ancestral sampling from our discrete-time reverse process for generation. %
We have found this yields slightly higher sample quality compared to other methods such as Euler discretization~\citep{campbell2022continuous,lou2023discrete}. 
For conditional generation tasks such as infilling, we find that the simple approach works best ---  we keep the conditioning tokens unmasked throughout the generation process. 
A complete description of the sampling algorithm can be found in \Cref{alg:sampling} in Appendix. 

\paragraph{Impact of schedules and discretization. }

For comparing different sampling configurations, we primarily use the FID score~\citep{heusel2017gans} on image datasets as our evaluation metric. We favor it over text generative perplexity\footnote{Perplexity of generated samples scored by a large language model such as GPT-2.} used in prior work~\citep{lou2023discrete}, as the latter can be misleadingly reduced by lowering sample diversity~\citep{holtzman2019curious}.
We initially trained our model using the linear schedule, which achieves the best final ELBO overall; however, we found that sampling did not perform well with a standard uniform discretization grid $t(i) = \frac{i}{T}$.
We hypothesize that time discretization can lead to conflicts by generating multiple tokens in a single step.
We then switched to the cosine schedule (\Cref{tab:noise-schedules}) that slows down unmasking at the beginning of reverse process. 
This drastically improves the FID on ImageNet 64$\times$64 from 70 to 17 for $T=256$ steps (\Cref{fig:imagenet-fid}, left). 
Building on this observation, we suggest using a ``cosine'' discretization grid for sampling in models trained with a linear schedule:
\begin{align}
    t(i) = \cos\Big(\frac{\pi}{2} \big(1 - \frac{i}{T}
     \big)\Big).
\end{align}
This induces the same discretization in $\alpha_t$ as the cosine schedule with a uniform grid, leading to comparable sample quality, as shown in Fig. 2 (left). %
In \Cref{fig:imagenet-fid} (right), we plot the number of tokens unmasked per step for linear and cosine schedules with a uniform grid. We believe the cosine schedule performs better because it leverages information redundancy: with more tokens revealed, the remaining tokens become more predictable, reducing conflicts when unmasking them in a single step. 

\begin{figure}[t]
    \centering
    \includegraphics[width=0.46\textwidth]{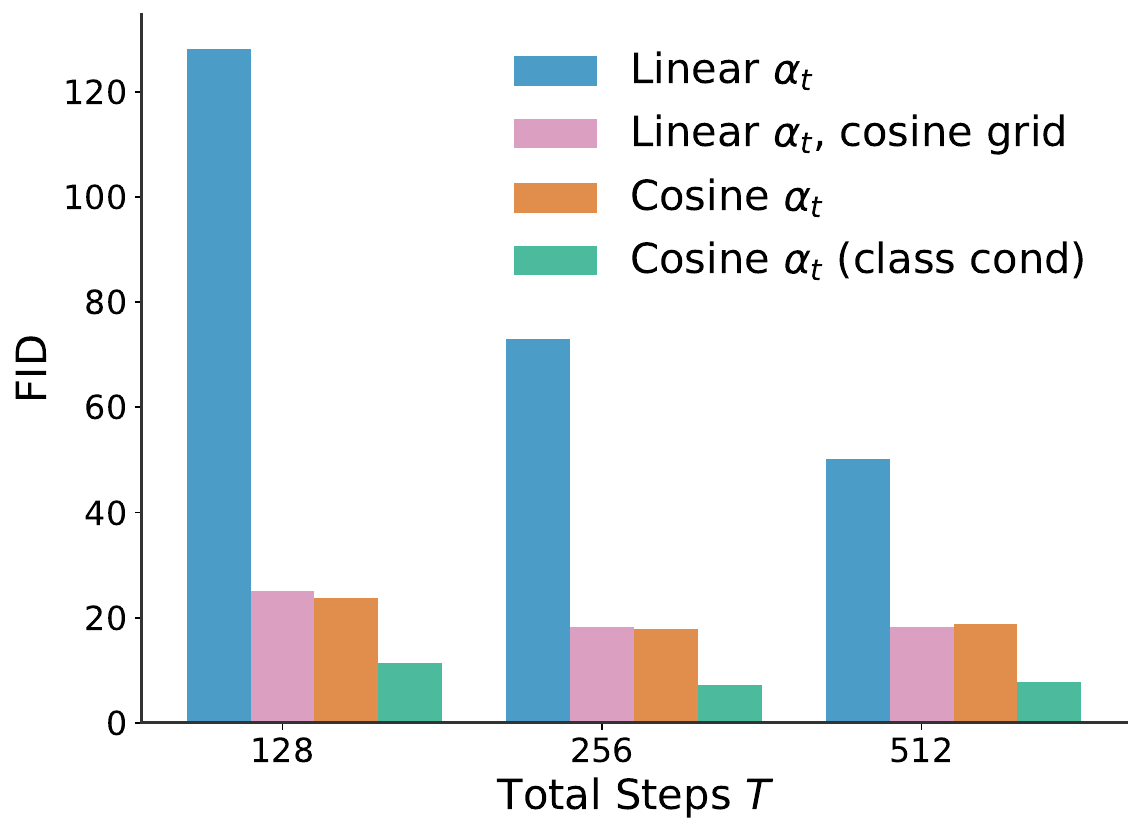}
    \includegraphics[width=0.46\textwidth]{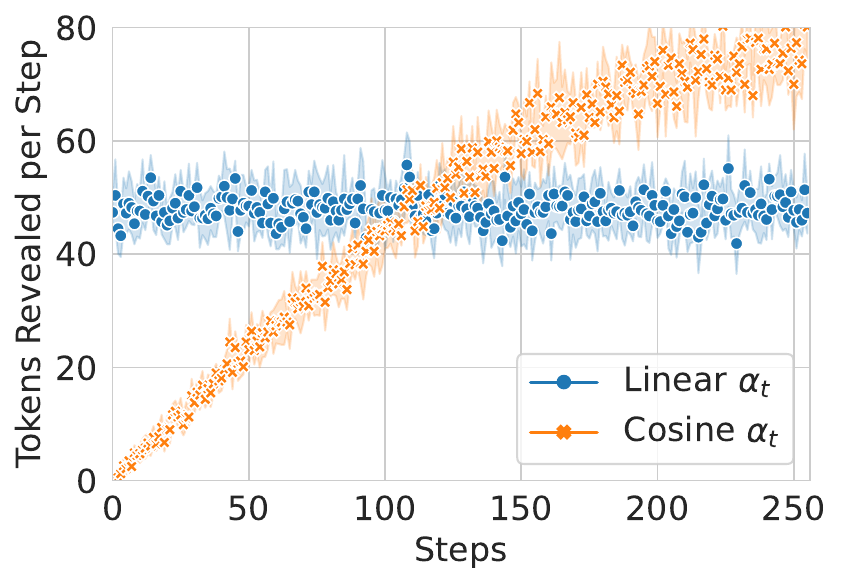}
    \caption{Left: FID evaluation for 50k samples randomly generated from MD4 on pixel-level modeling of ImageNet 64$\times$64 (numbers in \Cref{tab:imagenet-fid}). Right: Number of tokens revealed per generation step ($T=256$). Each image consists of $64\times 64\times 3 = 12288$ tokens. } %
    \label{fig:imagenet-fid}
\end{figure}

Although these findings were originally developed on images, we find them translate well to text~(see \Cref{fig:conditional-sample-generations}). 
we expect other techniques such as top-$p$ sampling~\citep{holtzman2019curious}, classifier-free guidance~\citep{ho2022classifier,nisonoff2024unlocking}, and predictor-correctors~\citep{campbell2022continuous,zhao2024informed} to further improve sample quality of our models. 
While we reserve these for future work, we note that the JAX~\citep{jax2018github} implementation of categorical sampling implicitly truncates small probabilities, creating a similar effect to top-$p$ sampling. See \Cref{app:jax-implicit-topp} for details.

\section{Relation to Existing Work}
\label{sec:related}

We discuss how to unify several existing masked diffusion models using our framework.  

\paragraph{Continuous-Time Markov Chains (CTMC).} 
To show the connection with the CTMC view presented in \citet{austin2021structured,campbell2022continuous}, 
we can %
write out the forward and reverse masked diffusion using CTMC machinery.
To see this, for a short time $\Delta t$, 
given $x_0$, the Taylor expansions of our forward and reverse transition matrices at $t$ are
\begin{align}
    \bar{Q}(t, t + \Delta t) &= I + Q(t)\Delta t + o(\Delta t) \text{\quad for\quad  } Q(t) \triangleq \beta(t)(\mathbf{1}e_m^\top - I), \\
    \bar{R}^{x_0}(t, t - \Delta t) &= I + R^{x_0}(t) \Delta t + o(\Delta t) \text{\quad for\quad } R^{x_0}(t) \triangleq -\frac{\alpha_t'}{1-\alpha_t} e_m(x_0 - e_m)^\top, \label{eq:rev-rate-x0}
\end{align}
where $Q(t)$ and $R^{x_0}(t)$ are known as the \emph{transition rate} matrices. 
\citet{austin2021structured} derived the same $Q(t)$ in App. A.6 of their paper. 
However, they did not explore the reverse process or a continuous-time objective. \citet{campbell2022continuous} 
derived an alternative ELBO expression using rate matrices, which \citet{kitouni2023disk} further simplified for absorbing diffusion.
In \Cref{app:ctmc}, we show how to recover their expression
by separating out a constant from our ELBO  expression \eqref{eq:diff-kl-limit} and applying a discrete ``integration-by-part''.
A key limitation of their expression
is that it needs $N$ evaluations of the prediction model $\mu_\theta(\cdot, t)$ to compute an inner summation. %
To circumvent this computational burden, they %
used a doubly stochastic estimate.  However, this leads to significantly higher variance compared to the analytic cross-entropy \eqref{eq:diff-kl-limit} which only requires one pass of $\mu_\theta(\cdot, t)$. 
Please refer to \Cref{app:campbell-diff} for more details. 

\paragraph{Score parameterization.}
While so far we used a prediction model $\mu_\theta(x_t, t)$ for the mean of clean data given $x_t$ (i.e., mean parameterization), one can choose other ways of parameterizing the reverse model. 
\citet{benton2024denoising,lou2023discrete} proposed to parameterize the discrete ``score'' $s(x_t, t)_j \triangleq \frac{q_t(j)}{q_t(x_t)}$
and introduced a score-based loss for discrete diffusions.
In \Cref{app:score}, we provide an alternative derivation of their loss
which is %
simpler. 
We show the link between score and mean parameterizations through the following proposition. 

\begin{proposition}[Score Parameterization vs. Mean Parameterization] \label{prop:score-mean-relation}
    Let $q_t$ be the marginal distribution of the masked diffusion defined in \Cref{sec:masked_diffusion} at time $t$. The discrete score $s(x_t, t)_j = \frac{q_t(j)}{q_t(x_t)}$ for a mask state $x_t = m$ and $j\neq m$ can be expressed as
    \begin{align} \label{eq:score-mean-rel}
        s(m, t)_j = \frac{\alpha_t}{1 - \alpha_t} \E[x_0|x_t=m]^\top e_j \text{, which satisfies } \sum_{j\neq m} s(m, t)_j = \frac{\alpha_t}{1 - \alpha_t}.
    \end{align}
\end{proposition}
\Cref{prop:score-mean-relation} (proved in \Cref{app:score}) implies that a reasonable %
score model for a mask state %
is 
\begin{align} \label{eq:score-mean-param}
    s_\theta(m,t)_j = \frac{\alpha_t}{1 - \alpha_t} \mu_\theta(m, t)_j. 
\end{align}
Indeed, substituting \eqref{eq:score-mean-param} into the score-based loss of \citet{lou2023discrete,benton2024denoising}
recovers our objective \eqref{eq:diff-kl-limit}. 
In \citet{lou2023discrete}, the score is parameterized as a neural network without enforcing the constraint in \eqref{eq:score-mean-rel}. 
This means the learned reverse model can be incompatible with the forward process. 
We find that our parameterization, which enforces the constraint, leads to more stable training and better results.

\paragraph{Any-order autoregressive models.} 
The continuous-time reverse process of our masked diffusion model can be viewed as an any-order autoregressive model (AO-ARM)~\citep{uria2014deep}. 
To see this, we reorder the tokens according to the timing of their unmasking events in the reverse process. 
For all tokens, the cumulative distribution functions (CDFs) of unmasking times $\{\tau_n\}_{n=1}^N$ are identical and satisfy
$P(\tau_n \leq t) = P(x_t^{(n)} = m) = 1 - \alpha_t$. 
As a result, the ordering is uniformly random across all possible arrangements, and the token prediction during each unmasking event represents a prediction step in AO-ARMs.
This connection was initially pointed out in \citet[][App. C]{hoogeboom2021autoregressive}. 
The relation between our simplified ELBO \eqref{eq:multi-diff-kl-limit} and the AO-ARM objective is independently clarified by \citet{ou2024}. 
Despite this equivalence, our work 
demonstrates that the masking schedule $\alpha_t$ introduces a new degree of freedom in the design of such models. 
Variations in $\alpha_t$ can lead to different distributions of unmasking times, significantly impacting performance in diffusion-style parallel sampling under time discretization, as shown in \Cref{fig:imagenet-fid}. 

\paragraph{Other related work.} Due to space constraint, we defer the discussion on other related work, including MaskGIT~\citep{chang2022maskgit}, discrete flow matching~\citep{campbell2024generative}, SDDM~\citep{sunscore}, Blackout diffusion~\citep{santos2023blackout} and SUNDAE~\citep{savinov2021step}, to \Cref{app:other-related}.

\section{Generalization to State-dependent Masking Schedules
\label{sec:dependent_rates}
}

Consider a scenario where some tokens hold more significance than others and we would like to unmask them earlier in the process. To achieve this, we introduce state-dependent masking schedules, where the probability of unmasking a token depends not only on time, but also on the token's value.

We first define the forward process for a single token $x_t$.
Let $\alpha_t$ be a $m +1$ dimensional vector function, i.e.,  
there is a different function $\alpha_{t,i}$
for each possible value $i$ of the token $x_t$. Also, by vector $\frac{\alpha_t}{\alpha_s}$ we denote the element-wise division
of the two vectors.  
We define the forward transition as 
$q(x_t|x_s) = \mathrm{Cat}(x_t; \bar{Q}(s, t)^\top x_s)$ where
\begin{align*}
    \bar{Q}(s, t) = \diag\Big(\frac{\alpha_t}{\alpha_s}\Big)  + \Big(I - \diag\Big(\frac{\alpha_t}{\alpha_s}\Big)\Big)\mathbf{1}e_m^\top
\end{align*}
and $\diag\big(\frac{\alpha_t}{\alpha_s}\big)$ is a diagonal matrix with the vector $\frac{\alpha_t}{\alpha_s}$ in its diagonal.
The probability of moving from current state $x_s$ to a future state $x_t$ (either the same as $x_s$ or mask) is determined by a state-dependent rate $\big( \frac{\alpha_t}{\alpha_s} \big)^\top x_s$, while 
the marginal at time $s$ given $x_0$ is 
$$
q(x_s|x_0) = \mathrm{Cat}(x_s; \bar{Q}(s)^\top x_0) \text{\quad for\quad } \bar{Q}(s) = \diag(\alpha_s) + (I - \diag(\alpha_s))\mathbf{1}e_m^\top.
$$
Further, for any time
$0\leq s < t \leq 1$ it holds that $q(x_t|x_0) = \sum_{x_s} q(x_t|x_s) q(x_s|x_0)$
so the above is a valid continuous-time Markov chain. 
Given the forward conditionals and marginals, we can now compute the time reversal conditioned on $x_0$.
The full form of $q(x_s|x_t, x_0)$ is derived in \Cref{app:state-dependentfirst}. 
For  $x_t= m$, we have 
\begin{talign}
q(x_s | x_t = m, x_0) = 
q(x_s | x_t = m, \textcolor{blue}{x_0}, \textcolor{red}{x_0 x_0^\top})
= \Big(\frac{{\bf 1} - \alpha_s
}
{ {\bf 1} -  \alpha_t}\Big)^\top \textcolor{blue}{x_0} e_m^\top x_s 
+ \Big(\frac{\alpha_s - \alpha_t}
{ {\bf 1} -  \alpha_t}\Big)^\top \textcolor{red}{x_0 x_0^\top} x_s. 
\label{eq:xsxt_em_x_0}
\end{talign}
This suggests that the reverse model 
 given $x_t = m$ can be chosen as
$p_\theta(x_s | x_t = m) 
\triangleq q(x_s | x_t = m, \textcolor{blue}{\mu_\theta(x_t,t)}, \textcolor{red}{\diag(\mu_\theta(x_t,t))})$
where $\mu_\theta(x_t,t)$ %
is a neural network
that approximates $\E[x_0 | x_t]$ while $\diag(\mu_\theta(x_t,t))$ approximates $\E[x_0 x_0^\top | x_t] = \diag(\E[x_0 | x_t])$. 
We show in \Cref{app:state-dependentfirst} that the negative continuous-time ELBO 
for the state-dependent rate case is
\begin{align} 
    \mathcal{L}_\infty 
    = \int_0^1 \Big(\frac{\alpha_t'}{\mathbf{1} - \alpha_{t}}\Big)^\top \E_{q(x_{t}|x_0)}\left[\delta_{x_t, m}\cdot (x_0 - \mu_\theta(x_t,t) + x_0x_0^\top \log \mu_\theta(x_{t},t))\right] \rmd t. 
    \label{eq:vec-diff-kl-limit}
\end{align}
Here, $\alpha_t'$ is the elementwise derivative of $\alpha_t$. 
This generalizes the MD4 loss \eqref{eq:diff-kl-limit}, which is recovered when $\alpha_t$ is a scalar schedule times a vector of ones.  
For $\numtokens$ tokens, the model further generalize similarly to  
\Cref{sec:model_objective} and the loss is given in \eqref{eq:genmd4-n}. 
We call this generalized model \textbf{GenMD4}.

\begin{figure}[t]
    \centering
    \includegraphics[trim={0 2cm 0 2cm},clip,width=0.96\textwidth]{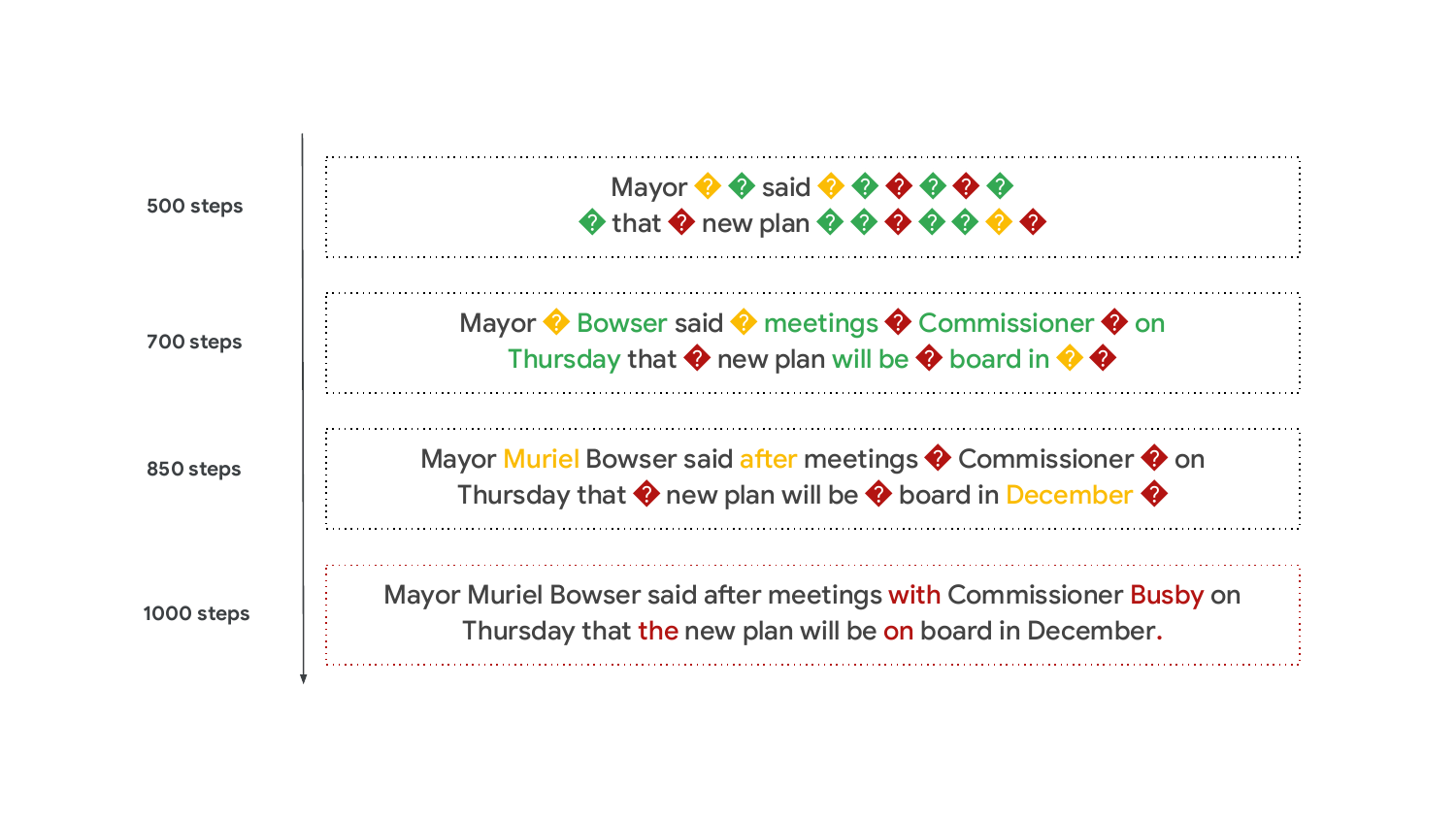}
    \caption{Iterative unmasking process for an unconditionally generated sample by MD4. This visualization %
    only includes a subsequence %
    from a generated sequence of 1024 tokens. %
    "?" %
    represents masks. %
    Masked tokens are revealed sequentially: green (steps 500-700), yellow (700-850), and red (850-1000).
    Additional unconditional generation from MD4 can be found in \Cref{app:owt-uncond}. 
    }
    \label{fig:sampling-progression}
\end{figure}

To learn the token dependent masking schedule using ELBO optimization, we parametrize the $m+1$ dimensional function  $\alpha_t$ using the polynomial schedule (see \Cref{fig:noise-schedules}) as
$
\alpha_{t, i} 
= 1 - t^{w_i}
$ 
and optimize each parameter $w_i > 0$.\footnote{We only need $m$ learnable parameters $w_i$, for $i=0,\ldots,m-1$, 
since $x_0$ can never be the mask token. For the final mask  dimension we can choose an arbitrary fixed value such as $w_m=0$.} The value of $w_i$, through the 
masking probability $1- \alpha_{t,i}$, %
determines how fast the token  with value $i$ jumps to the %
mask state. 
Since in the loss \eqref{eq:vec-diff-kl-limit} the distribution
$q(x_t|x_0)$ depends on $\alpha_t$ and thus the vector $w$, optimizing $w$ poses a discrete gradient estimation problem~\citep[see, e.g., ][]{shi2022gradient}. 
Naive autodiff leads to biased gradients and pushes $w$ towards zero because the gradients cannot propagate through the (discrete) samples drawn from $q(x_t|x_0)$. 
To fix this, we used the REINFORCE leave-one-out estimator~\citep{salimans2014using,Kool2019Buy4R} to compute low-variance unbiased gradients for optimizing $w$. 
Details are given in \Cref{app:grad-est}.

\section{Experiments}

\subsection{Text}

Text is natural discrete data with rich structures. 
For comparison with prior work, we evaluate likelihood on two datasets:
\textbf{text8}~\citep{text8}, a character-level text modeling benchmark, %
and \textbf{OpenWebText} \citep{Gokaslan2019OpenWeb}, an open clone of the unreleased WebText dataset used to train GPT-2~\citep{radford2019language}.
We also assess our model's performance on downstream tasks by training on
\textbf{FineWeb-Edu}~\citep{penedo2024fineweb}, a high-quality dataset of fine educational text commonly used by the open-source community for comparing LLMs.
Unless otherwise specified, a linear schedule and a cosine sampling grid are employed.

\begin{table}[t]
\footnotesize
\centering %
\caption{Zero-shot unconditional perplexity on five benchmark datasets from \citet{radford2019language}. The numbers for other methods are from \citet{lou2023discrete} except our reimplementation of SEDD Absorb. 
Our MD4 model 
achieves the best result on all benchmarks except LAMBADA where it is the second best. $^*$The GPT-2 numbers are reported for the GPT-2 checkpoint pretrained on WebText instead of OWT thus is not a direct comparison.}
\label{tab:model_performance}
\vskip 0.1in
\setlength{\tabcolsep}{5pt}
\resizebox{\textwidth}{!}{%
\begin{tabular}{llrrrrr}
Size & Method %
& LAMBADA & WikiText2 & PTB & WikiText103 & IBW \\
\midrule %
Small & GPT-2 (WebText)$^*$ %
& \bf 45.04 & 42.43 & 138.43 & 41.60 & 75.20 \\
    & D3PM %
    & $\le$ 93.47 & $\le$ 77.28 & $\le$ 200.82 & $\le$ 75.16 & $\le$ 138.92\\
    & Plaid %
    & $\le$ 57.28 & $\le$ 51.80 & $\le$ 142.60 & $\le$ 50.86 & $\le$ 91.12\\ 
      & SEDD Absorb %
      & $\le$ 50.92 & $\le$ 41.84 & $\le$ 114.24 & $\le$ 40.62 & $\le$ 79.29 \\
      & SEDD Absorb (reimpl.) %
      & $\le$ 49.73 & $\le$ 38.94 & $\le$ 107.54 & $\le$ 39.15 & $\le$ 72.96 \\
      & MD4 (Ours) %
      & $\le$ 48.43 & $\le$ \textbf{34.94} & $\le$ \bf 102.26 & $\le$ \textbf{35.90} & $\le$ \bf 68.10 \\ 
\midrule
Medium & GPT-2 (WebText)$^*$ %
       & \textbf{35.66} & 31.80 & 123.14 & 31.39 & 55.72 \\
       & SEDD Absorb %
       & $\le$ 42.77 & $\le$ 31.04 & $\le$ 87.12 & $\le$ 29.98 & $\le$ 61.19 \\
       & MD4 (Ours) %
       & $\le$ 44.12 & $\le$ \textbf{25.84} & $\le$ \textbf{66.07} & $\le$ \textbf{25.84} & $\le$ \textbf{51.45} \\
\bottomrule %
\label{tab:owt-zeroshot-ppl}
\end{tabular}}
\vskip -0.15in
\end{table}

\begin{wrapfigure}{r}{0.5\textwidth}
    \vskip -0.35in
    \centering
    \includegraphics[trim={0 0 2cm 0},clip,width=0.48\textwidth]{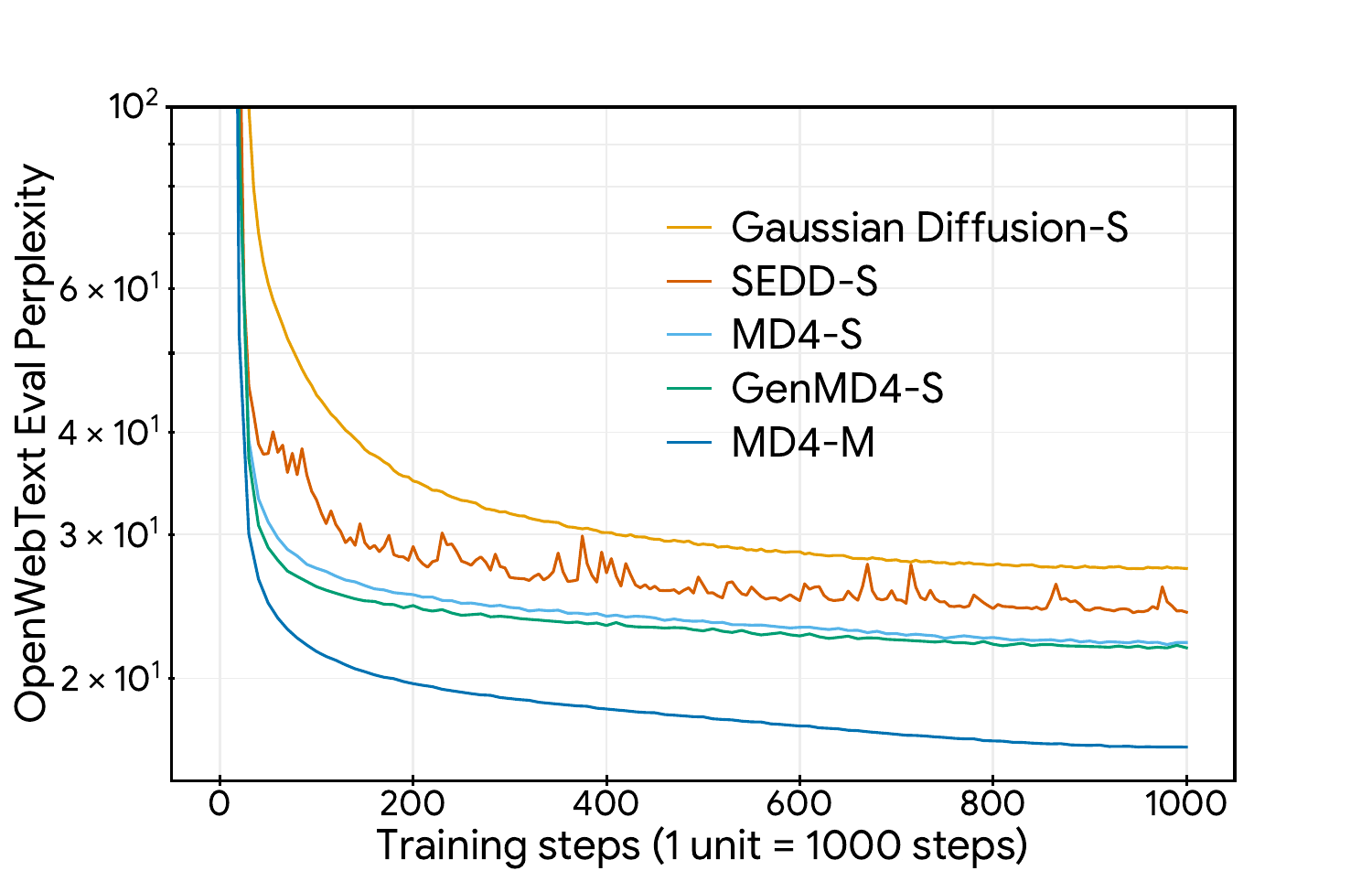}
    \caption{Perplexity on OpenWebText (OWT) validation set during training. The final numbers are reported in \Cref{tab:owt-eval-ppl} in Appendix.}
    \label{fig:owt-eval}
\end{wrapfigure}

\paragraph{OpenWebText.}
We train MD4 of GPT-2 small (S) and GPT-2 medium (M) sizes on OpenWebText %
and evaluate zero-shot perplexity %
on 
five benchmark datasets used in \citet{radford2019language}. 
We keep our evaluation setup the same as SEDD \citep{lou2023discrete}. 
To ensure fair comparison, we %
reimplemented SEDD in our codebase. 
Our implementation led to slightly better results than those  
reported in their paper.

As seen in 
\Cref{tab:owt-zeroshot-ppl}, our small model outperforms %
previous best discrete diffusion models %
on all five tasks. 
We are also better than GPT-2 on all tasks except LAMBADA where we %
are the second best method.
When scaling up to medium size, MD4 similarly beats SEDD and GPT-2 on 4 out of 5 tasks.

To confirm that the strong zero-shot performance stems from improved training, we plot perplexity on 2\% OpenWebText validation set in \Cref{fig:owt-eval}. 
Our models converge faster and have better final likelihoods than prior methods. 
We also observed that SEDD~\citep{lou2023discrete} has training instabilities, likely due to score parameterization breaking consistency between forward and reverse processes (\Cref{sec:related}).  
Although GenMD4 achieves lower perplexity than MD4, we observed that the learned $w$s can overfit to dataset statistics, making it less effective on zero-shot transfer tasks. 

We also assess our models' generation quality. 
\Cref{fig:sampling-progression} shows a randomly selected, notably coherent sample from MD4-medium and its denoising process. 
\Cref{fig:conditional-sample-generations} demonstrates MD4's text infilling ability and highlights a substantial quality gain when transitioning from uniform to cosine discretization (see \Cref{sec:sampling}).
Despite MD4's strong performance on quantitative metrics like generative perplexity, we have placed these results in Appendix \Cref{fig:generative-ppl} due to the metric's inherent unreliability, as noted in \Cref{sec:sampling}. 
We emphasize the more reliable FID-based assessments found in our image experiments.

\paragraph{Text8.} 
Following prior work \citep{austin2021structured,lou2023discrete}, we trained masked diffusion models on text8 %
and evaluate the bits-per-character on the test set
(details in \Cref{app:exp-text8}). As seen in \Cref{tab:text8}, our models
outperform previous discrete and continuous diffusion models, as well as state-of-the-art AO-ARMs which are closely related to discrete diffusion~\citep{hoogeboom2021autoregressive}.
Our model is only beaten by an autoregressive (AR) transformer and the AR-backbone Discrete Flow~\citep{tran2019discrete}. 
We believe this is because AR models only require learning a fixed generation order thus better utilize model capacity.
Text8's small vocabulary~(26 letters and a space) 
led us to expect limited flexibility from our state-dependent formulation. However, using the generalized objective in \eqref{eq:vec-diff-kl-limit}, GenMD4 achieved significantly better BPC than MD4, demonstrating the potential of state-dependent diffusion for discrete data.

\begin{wrapfigure}{r}{0.5\textwidth}
    \centering
    \includegraphics[clip,width=0.48\textwidth]{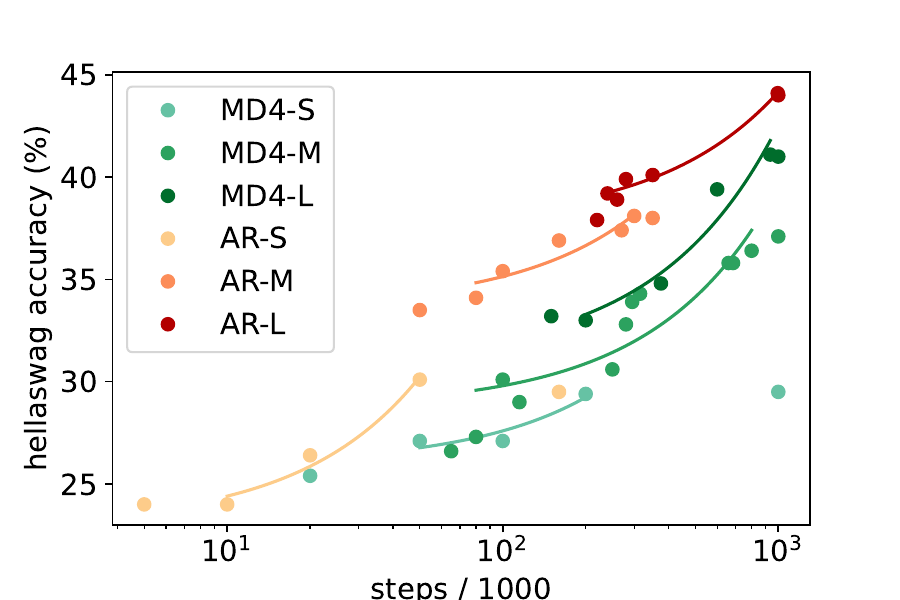}
    \caption{Hellaswag accuracy vs. training steps for MD4 and AR models at GPT-2 small, medium, and large scales.}
    \label{fig:hellaswag-eval}
\end{wrapfigure}

\paragraph{FineWeb-Edu.}  

We train MD4 on FineWeb-Edu and evaluate its zero-shot accuracy on the Hellaswag dataset \citep{zellers2019hellaswag}, a popular common sense inference benchmark for LLMs. We directly compared MD4 to its AR counterparts – transformers with identical configurations (except for causal masking) trained on the same data. Results are summarized in \Cref{fig:hellaswag-eval}.

MD4 demonstrates steady performance growth with increasing scale. While outperformed by AR models of the same size, the performance gap does not widen as model size increases. For example, AR-small reaches 30\% accuracy in 50k steps, while MD4-small takes 200k steps (4x data efficiency difference). At the medium scale, AR achieves 37\% in 270k steps, compared to MD4's 1 million steps.

\begin{table}[t] \vskip -0.1in
\footnotesize
\parbox{.42\textwidth}{
    \centering
    \caption{Bits Per Character (BPC) on Text8 test set. All models use standard 12-layer transformers similar to GPT-2 small~\citep{radford2019language} except Discrete Flow which uses $8\times 3$ layers. }
    \label{tab:text8}
    \vskip 0.1in
    \begin{tabular}{lr}
    Method &  BPC ($\downarrow$) \\ \midrule
    \textit{Continuous Diffusion} & \\
    Plaid \citep{gulrajani2024likelihood} (Our impl.) & $\le$ 1.48 \\
    BFN \citep{graves2023bayesian} & $\le$ 1.41\\
    \midrule
    \textit{Any-order Autoregressive} \\
    ARDM \citep{hoogeboom2021autoregressive} & $\le$ 1.43 \\
    MAC \citep{shih2022training} & $\le$ 1.40\\
    \midrule
    \textit{Autoregressive} & \\
    IAF/SCF~\citep{ziegler2019latent}  & 1.88\\
    AR Argmax Flow~\citep{hoogeboom2021argmax} & 1.39\\
    Discrete Flow~\citep{tran2019discrete} & \textbf{1.23}\\
    Transformer AR~\citep{austin2021structured} & \textbf{1.23}\\
    \midrule
    \textit{Discrete Diffusion} & \\
    Mult. Diffusion \citep{hoogeboom2021argmax} & $\le$ 1.72\\
    D3PM Uniform \citep{austin2021structured} & $\le$ 1.61\\
    D3PM Absorb \citep{austin2021structured} & $\le$ 1.45\\ 
    SEDD Absorb \citep{lou2023discrete} & $\le$ 1.39 \\ %
    MD4 (Ours) & $\le$ \textbf{1.37} \\ 
    GenMD4 (Ours) & $\le$ \textbf{1.34} \\
    \bottomrule
    \end{tabular} \vskip -0.1in
}
\hfill
\parbox{.5\textwidth}{
    \centering %
    \caption{Bits Per Dimension (BPD) on CIFAR-10 test set and Downsampled ImageNet 64$\times$64~\citep{van2016pixel} validation set. All models in the table are trained without data augmentation.}
    \label{tab:cifar}
    \vskip 0.05in
    \setlength{\tabcolsep}{5pt}
\resizebox{\linewidth}{!}{%
    \begin{tabular}{llrr}
     & Method & \#Params &  BPD ($\downarrow$) \\
    \midrule %
    \multirow{13}{*}{ \rotatebox{90}{CIFAR-10} } 
    & \textit{Autoregressive} \\
    & PixelRNN \citep{van2016pixel} & &  3.00 \\
    & Gated PixelCNN \citep{van2016conditional} &  & 3.03 \\
    & PixelCNN++ \citep{salimans2016pixelcnn} & 53M  & 2.92 \\
    & PixelSNAIL \citep{chen2018pixelsnail} & 46M & 2.85 \\
    & Image Transformer \citep{parmar2018image} &  & 2.90 \\
    & Sparse Transformer \citep{child2019generating} & 59M  & 2.80 \\
    \cmidrule{2-4}
    & \textit{Discrete Diffusion} \\
    & D3PM Absorb \citep{austin2021structured} & 37M & $\le$ 4.40 \\
    & D3PM Gauss \citep{austin2021structured} & 36M & $\le$ 3.44  \\ 
    & \citet{campbell2022continuous} & 36M & $\le$ 3.59 \\
    & \citet{campbell2022continuous} Absorb & 28M & $\le$ 3.52 \\
    & MD4 (Ours) & 28M & $\le$ \bf 2.75 \\
    \midrule %
    \multirow{8}{*}{ \rotatebox[origin=c]{90}{ImageNet 64$\times$64} } & \textit{Autoregressive} \\
    & PixelRNN \citep{van2016pixel} & &  3.63 \\
    & Gated PixelCNN \citep{van2016conditional} &  & 3.57 \\
    & Sparse Transformer \citep{child2019generating} &  152M & 3.44 \\
    & Routing Transformer \citep{roy2021efficient} & & 3.43 \\
    & Perceiver AR \citep{child2019generating} & 770M  & \bf 3.40 \\
    \cmidrule{2-4}
    & \textit{Discrete Diffusion} \\
    & MD4 (Ours) & 198M &  $\le$ \bf 3.40 \\
    \bottomrule %
    \end{tabular}}\vskip -0.1in
}

\end{table}

\subsection{Pixel-level image modeling} 

\begin{figure}[htb]
    \centering
    \includegraphics[width=0.9\textwidth]{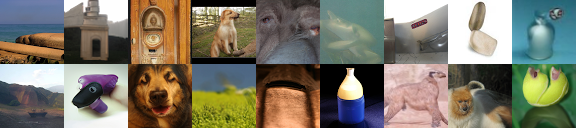}
    \caption{Non cherry-picked unconditional samples from MD4 trained on ImageNet 64x64, treating pixels as discrete tokens. More samples can be found in \Cref{fig:imagenet64-app} in Appendix. The model is optimized for likelihood instead of visual quality---see e.g., \citet{kingma2021variational} for samples from a continuous diffusion model optimized similarly for likelihood. }
    \label{fig:imagenet64}
\end{figure}

Unlike continuous diffusion which struggles with discrete data, we show that MD4, a discrete diffusion model, performs well on inherently continuous data, suggesting its potential for unifying modalities. 
We follow \citet{austin2021structured} and train MD4 on order-agnostic image data from CIFAR-10 and downsampled ImageNet 64$\times$64~\citep{van2016pixel}. 
Each image is treated as a set of 256-valued discrete tokens, making the model agnostic to pixel proximity.
We compare to other discrete diffusion and AR models with reported likelihood results on these datasets, although to our knowledge there are no published result on discrete diffusion for ImageNet $64\times 64$ that directly model raw pixel space. 

\Cref{tab:cifar} summarizes our results. 
We establish a new state-of-the-art for discrete diffusion models, %
outperforming previous work \citep{austin2021structured,campbell2022continuous} by a significant margin. 
Our CIFAR-10 result surpasses the best reported AR result. 
On ImageNet $64\times 64$, our results are competitive with Transformer AR models that are 4$\times$ larger, as well as a strong continuous diffusion model VDM \citep{kingma2021variational}.
Notably, despite lacking knowledge of the ordinal structure of pixel values, MD4 outperforms models trained with this inductive bias, including D3PM Gauss and \citet{campbell2022continuous} where the noising distribution is a discrete Gaussian that assigns larger probabilities to near pixel values. 
To isolate the differences caused by training objectives, we also implemented the \citet{campbell2022continuous} objective with the absorbing process, showing its high variance hinders learning even with our architecture.

We provide a random sample from our ImageNet 64$\times$64 model in \Cref{fig:imagenet64}. 
More results can be found in \Cref{app:results}.
In \Cref{fig:imagenet-fid}, we plot the FID values of samples generated under different choices of schedules and discretization grids. 
We can see that the model with the linear schedule plus a cosine grid achieves an FID close to the model with cosine schedule, both significantly outperform the linear schedule with a uniform grid. 
We further trained a class-conditional model on ImageNet 64$\times$64 that boosts the FID to around 7.
Although these are not state-of-the-art FIDs on ImageNet 64$\times$64, we emphasize our models are optimized for likelihood instead of sample quality.

\section{Conclusion}
\vspace{-2mm} 

In this work, we revisit masked diffusion models, focusing on a flexible continuous-time formulation. Existing works in this area are not easily accessible to non-specialists and present ELBOs that are difficult to optimize, often resulting in performance that is not competitive with continuous diffusions and AR models. The framework we propose provides a very simple expression of the ELBO as a weighted integral of cross-entropy losses. Additionally, we propose a generalized masked diffusion formulation (GenMD4), where the masking schedule depends on the current state of the process, and derive its corresponding ELBO. %
On text data, our MD4 models outperform existing discrete and continuous diffusion models. For pixel-level image modeling, we significantly improve discrete diffusion results, outperforming similar-sized AR models and achieving comparable likelihoods to continuous diffusion models such as VDM. 
GenMD4 provides further improvements in terms of likelihoods over the state-independent case.

Although we have improved masked diffusion models, they still suffer from limitations. First, in some tasks such as text8, masked diffusions are not yet competitive with AR models. We conjecture that this is because AR models can better leverage model capacity since they only require learning one order. It would be interesting to develop better architectures for discrete diffusions. Moreover, GenMD4 is promising, but it can easily overfit to the dataset, making it less effective for zero-shot transfer compared to simpler versions. Additionally, inference with a state-dependent schedule is more challenging.
{
\small

\bibliographystyle{unsrtnat}
\bibliography{ref}

\begin{thebibliography}{75}
\providecommand{\natexlab}[1]{#1}
\providecommand{\url}[1]{\texttt{#1}}
\expandafter\ifx\csname urlstyle\endcsname\relax
  \providecommand{\doi}[1]{doi: #1}\else
  \providecommand{\doi}{doi: \begingroup \urlstyle{rm}\Url}\fi

\bibitem[Sohl-Dickstein et~al.(2015)Sohl-Dickstein, Weiss, Maheswaranathan, and
  Ganguli]{sohl2015deep}
Jascha Sohl-Dickstein, Eric Weiss, Niru Maheswaranathan, and Surya Ganguli.
\newblock Deep unsupervised learning using nonequilibrium thermodynamics.
\newblock In \emph{International Conference on Machine Learning}, 2015.

\bibitem[Ho et~al.(2020)Ho, Jain, and Abbeel]{ho2020denoising}
Jonathan Ho, Ajay Jain, and Pieter Abbeel.
\newblock Denoising diffusion probabilistic models.
\newblock In \emph{Advances in Neural Information Processing Systems}, 2020.

\bibitem[Song et~al.(2020)Song, Sohl-Dickstein, Kingma, Kumar, Ermon, and
  Poole]{song2020score}
Yang Song, Jascha Sohl-Dickstein, Diederik~P Kingma, Abhishek Kumar, Stefano
  Ermon, and Ben Poole.
\newblock Score-based generative modeling through stochastic differential
  equations.
\newblock In \emph{International Conference on Learning Representations}, 2020.

\bibitem[Rombach et~al.(2022)Rombach, Blattmann, Lorenz, Esser, and
  Ommer]{rombach2022high}
Robin Rombach, Andreas Blattmann, Dominik Lorenz, Patrick Esser, and Bj{\"o}rn
  Ommer.
\newblock High-resolution image synthesis with latent diffusion models.
\newblock In \emph{Proceedings of the IEEE/CVF Conference on Computer Vision
  and Pattern Recognition}, pages 10684--10695, 2022.

\bibitem[Ramesh et~al.(2022)Ramesh, Dhariwal, Nichol, Chu, and
  Chen]{ramesh2022hierarchical}
Aditya Ramesh, Prafulla Dhariwal, Alex Nichol, Casey Chu, and Mark Chen.
\newblock Hierarchical text-conditional image generation with clip latents.
\newblock \emph{arXiv preprint arXiv:2204.06125}, 1\penalty0 (2):\penalty0 3,
  2022.

\bibitem[Saharia et~al.(2022)Saharia, Chan, Saxena, Li, Whang, Denton,
  Ghasemipour, Gontijo~Lopes, Karagol~Ayan, Salimans,
  et~al.]{saharia2022photorealistic}
Chitwan Saharia, William Chan, Saurabh Saxena, Lala Li, Jay Whang, Emily~L
  Denton, Kamyar Ghasemipour, Raphael Gontijo~Lopes, Burcu Karagol~Ayan, Tim
  Salimans, et~al.
\newblock Photorealistic text-to-image diffusion models with deep language
  understanding.
\newblock In \emph{Advances in Neural Information Processing Systems}, 2022.

\bibitem[Chen et~al.(2021)Chen, Zhang, Zen, Weiss, Norouzi, and
  Chan]{chen2021wavegrad}
Nanxin Chen, Yu~Zhang, Heiga Zen, Ron~J Weiss, Mohammad Norouzi, and William
  Chan.
\newblock Wavegrad: Estimating gradients for waveform generation.
\newblock In \emph{International Conference on Learning Representations}, 2021.

\bibitem[Kong et~al.(2021)Kong, Ping, Huang, Zhao, and
  Catanzaro]{kong2021diffwave}
Zhifeng Kong, Wei Ping, Jiaji Huang, Kexin Zhao, and Bryan Catanzaro.
\newblock Diffwave: A versatile diffusion model for audio synthesis.
\newblock In \emph{International Conference on Learning Representations}, 2021.

\bibitem[Ho et~al.(2022)Ho, Salimans, Gritsenko, Chan, Norouzi, and
  Fleet]{ho2022video}
Jonathan Ho, Tim Salimans, Alexey Gritsenko, William Chan, Mohammad Norouzi,
  and David~J Fleet.
\newblock Video diffusion models.
\newblock In \emph{Advances in Neural Information Processing Systems}, 2022.

\bibitem[Villegas et~al.(2023)Villegas, Babaeizadeh, Kindermans, Moraldo,
  Zhang, Saffar, Castro, Kunze, and Erhan]{villegas2023phenaki}
Ruben Villegas, Mohammad Babaeizadeh, Pieter-Jan Kindermans, Hernan Moraldo,
  Han Zhang, Mohammad~Taghi Saffar, Santiago Castro, Julius Kunze, and Dumitru
  Erhan.
\newblock Phenaki: Variable length video generation from open domain textual
  descriptions.
\newblock In \emph{International Conference on Learning Representations}, 2023.

\bibitem[Bar-Tal et~al.(2024)Bar-Tal, Chefer, Tov, Herrmann, Paiss, Zada,
  Ephrat, Hur, Li, Michaeli, et~al.]{bar2024lumiere}
Omer Bar-Tal, Hila Chefer, Omer Tov, Charles Herrmann, Roni Paiss, Shiran Zada,
  Ariel Ephrat, Junhwa Hur, Yuanzhen Li, Tomer Michaeli, et~al.
\newblock Lumiere: A space-time diffusion model for video generation.
\newblock \emph{arXiv preprint arXiv:2401.12945}, 2024.

\bibitem[OpenAI(2024)]{openai2024sora}
OpenAI.
\newblock Sora.
\newblock \url{https://openai.com/index/sora/}, 2024.

\bibitem[Bao et~al.(2024)Bao, Xiang, Yue, He, Zhu, Zheng, Zhao, Liu, Wang, and
  Zhu]{bao2024vidu}
Fan Bao, Chendong Xiang, Gang Yue, Guande He, Hongzhou Zhu, Kaiwen Zheng, Min
  Zhao, Shilong Liu, Yaole Wang, and Jun Zhu.
\newblock Vidu: a highly consistent, dynamic and skilled text-to-video
  generator with diffusion models.
\newblock \emph{arXiv preprint arXiv:2405.04233}, 2024.

\bibitem[Austin et~al.(2021)Austin, Johnson, Ho, Tarlow, and Van
  Den~Berg]{austin2021structured}
Jacob Austin, Daniel~D Johnson, Jonathan Ho, Daniel Tarlow, and Rianne Van
  Den~Berg.
\newblock Structured denoising diffusion models in discrete state-spaces.
\newblock In \emph{Advances in Neural Information Processing Systems}, 2021.

\bibitem[Hoogeboom et~al.(2021{\natexlab{a}})Hoogeboom, Nielsen, Jaini,
  Forr{\'e}, and Welling]{hoogeboom2021argmax}
Emiel Hoogeboom, Didrik Nielsen, Priyank Jaini, Patrick Forr{\'e}, and Max
  Welling.
\newblock Argmax flows and multinomial diffusion: Learning categorical
  distributions.
\newblock In \emph{Advances in Neural Information Processing Systems},
  2021{\natexlab{a}}.

\bibitem[Vignac et~al.(2023)Vignac, Krawczuk, Siraudin, Wang, Cevher, and
  Frossard]{vignac2023digress}
Cl{\'e}ment Vignac, Igor Krawczuk, Antoine Siraudin, Bohan Wang, Volkan Cevher,
  and Pascal Frossard.
\newblock {DiGress}: Discrete denoising diffusion for graph generation.
\newblock In \emph{International Conference on Learning Representations}, 2023.

\bibitem[Yang et~al.(2023)Yang, Yu, Wang, Wang, Weng, Zou, and
  Yu]{yang2023diffsound}
Dongchao Yang, Jianwei Yu, Helin Wang, Wen Wang, Chao Weng, Yuexian Zou, and
  Dong Yu.
\newblock Diffsound: Discrete diffusion model for text-to-sound generation.
\newblock \emph{IEEE/ACM Transactions on Audio, Speech, and Language
  Processing}, 2023.

\bibitem[Gruver et~al.(2023)Gruver, Stanton, Frey, Rudner, Hotzel,
  Lafrance-Vanasse, Rajpal, Cho, and Wilson]{gruver2024protein}
Nate Gruver, Samuel Stanton, Nathan Frey, Tim~GJ Rudner, Isidro Hotzel, Julien
  Lafrance-Vanasse, Arvind Rajpal, Kyunghyun Cho, and Andrew~G Wilson.
\newblock Protein design with guided discrete diffusion.
\newblock In \emph{Advances in Neural Information Processing Systems}, 2023.

\bibitem[Dieleman et~al.(2022)Dieleman, Sartran, Roshannai, Savinov, Ganin,
  Richemond, Doucet, Strudel, Dyer, Durkan, et~al.]{dieleman2022continuous}
Sander Dieleman, Laurent Sartran, Arman Roshannai, Nikolay Savinov, Yaroslav
  Ganin, Pierre~H Richemond, Arnaud Doucet, Robin Strudel, Chris Dyer, Conor
  Durkan, et~al.
\newblock Continuous diffusion for categorical data.
\newblock \emph{arXiv preprint arXiv:2211.15089}, 2022.

\bibitem[Chen et~al.(2022)Chen, ZHANG, and Hinton]{chen2022analog}
Ting Chen, Ruixiang ZHANG, and Geoffrey Hinton.
\newblock Analog bits: Generating discrete data using diffusion models with
  self-conditioning.
\newblock In \emph{International Conference on Learning Representations}, 2022.

\bibitem[Li et~al.(2022)Li, Thickstun, Gulrajani, Liang, and
  Hashimoto]{li2022diffusion}
Xiang Li, John Thickstun, Ishaan Gulrajani, Percy~S Liang, and Tatsunori~B
  Hashimoto.
\newblock Diffusion-{LM} improves controllable text generation.
\newblock In \emph{Advances in Neural Information Processing Systems}, 2022.

\bibitem[Gulrajani and Hashimoto(2023)]{gulrajani2024likelihood}
Ishaan Gulrajani and Tatsunori~B Hashimoto.
\newblock Likelihood-based diffusion language models.
\newblock In \emph{Advances in Neural Information Processing Systems}, 2023.

\bibitem[Lovelace et~al.(2024)Lovelace, Kishore, Wan, Shekhtman, and
  Weinberger]{lovelace2024latent}
Justin Lovelace, Varsha Kishore, Chao Wan, Eliot Shekhtman, and Kilian~Q
  Weinberger.
\newblock Latent diffusion for language generation.
\newblock In \emph{Advances in Neural Information Processing Systems}, 2024.

\bibitem[Richemond et~al.(2022)Richemond, Dieleman, and
  Doucet]{richemond2022categorical}
Pierre~H Richemond, Sander Dieleman, and Arnaud Doucet.
\newblock Categorical {SDE}s with simplex diffusion.
\newblock \emph{arXiv preprint arXiv:2210.14784}, 2022.

\bibitem[Avdeyev et~al.(2023)Avdeyev, Shi, Tan, Dudnyk, and
  Zhou]{avdeyev2023dirichlet}
Pavel Avdeyev, Chenlai Shi, Yuhao Tan, Kseniia Dudnyk, and Jian Zhou.
\newblock Dirichlet diffusion score model for biological sequence generation.
\newblock In \emph{International Conference on Machine Learning}, 2023.

\bibitem[Graves et~al.(2023)Graves, Srivastava, Atkinson, and
  Gomez]{graves2023bayesian}
Alex Graves, Rupesh~Kumar Srivastava, Timothy Atkinson, and Faustino Gomez.
\newblock Bayesian flow networks.
\newblock \emph{arXiv preprint arXiv:2308.07037}, 2023.

\bibitem[Xue et~al.(2024)Xue, Zhou, Nie, Min, Zhang, Zhou, and
  Li]{xue2024unifying}
Kaiwen Xue, Yuhao Zhou, Shen Nie, Xu~Min, Xiaolu Zhang, Jun Zhou, and Chongxuan
  Li.
\newblock Unifying {B}ayesian flow networks and diffusion models through
  stochastic differential equations.
\newblock \emph{arXiv preprint arXiv:2404.15766}, 2024.

\bibitem[Liu et~al.(2024)Liu, Chen, Theodorou, and Tao]{liu2024mirror}
Guan-Horng Liu, Tianrong Chen, Evangelos Theodorou, and Molei Tao.
\newblock Mirror diffusion models for constrained and watermarked generation.
\newblock In \emph{Advances in Neural Information Processing Systems}, 2024.

\bibitem[Campbell et~al.(2022)Campbell, Benton, De~Bortoli, Rainforth,
  Deligiannidis, and Doucet]{campbell2022continuous}
Andrew Campbell, Joe Benton, Valentin De~Bortoli, Thomas Rainforth, George
  Deligiannidis, and Arnaud Doucet.
\newblock A continuous time framework for discrete denoising models.
\newblock In \emph{Advances in Neural Information Processing Systems}, 2022.

\bibitem[Sun et~al.(2022)Sun, Yu, Dai, Schuurmans, and Dai]{sunscore}
Haoran Sun, Lijun Yu, Bo~Dai, Dale Schuurmans, and Hanjun Dai.
\newblock Score-based continuous-time discrete diffusion models.
\newblock In \emph{International Conference on Learning Representations}, 2022.

\bibitem[Zheng et~al.(2023)Zheng, Yuan, Yu, and Kong]{zheng2023reparameterized}
Lin Zheng, Jianbo Yuan, Lei Yu, and Lingpeng Kong.
\newblock A reparameterized discrete diffusion model for text generation.
\newblock \emph{arXiv preprint arXiv:2302.05737}, 2023.

\bibitem[Lou et~al.(2024)Lou, Meng, and Ermon]{lou2023discrete}
Aaron Lou, Chenlin Meng, and Stefano Ermon.
\newblock Discrete diffusion language modeling by estimating the ratios of the
  data distribution.
\newblock In \emph{International Conference on Machine Learning}, 2024.

\bibitem[Kingma et~al.(2021)Kingma, Salimans, Poole, and
  Ho]{kingma2021variational}
Diederik Kingma, Tim Salimans, Ben Poole, and Jonathan Ho.
\newblock Variational diffusion models.
\newblock In \emph{Advances in Neural Information Processing Systems}, 2021.

\bibitem[Karras et~al.(2022)Karras, Aittala, Aila, and
  Laine]{karras2022elucidating}
Tero Karras, Miika Aittala, Timo Aila, and Samuli Laine.
\newblock Elucidating the design space of diffusion-based generative models.
\newblock In \emph{Advances in Neural Information Processing Systems}, 2022.

\bibitem[Benton et~al.(2022)Benton, Shi, De~Bortoli, Deligiannidis, and
  Doucet]{benton2024denoising}
Joe Benton, Yuyang Shi, Valentin De~Bortoli, George Deligiannidis, and Arnaud
  Doucet.
\newblock From denoising diffusions to denoising {M}arkov models.
\newblock \emph{arXiv preprint arXiv:2211.03595}, 2022.

\bibitem[Ou et~al.(2024)Ou, Nie, Xue, Zhu, Sun, Li, and Li]{ou2024}
Jingyang Ou, Shen Nie, Kaiwen Xue, Fengqi Zhu, Jiacheng Sun, Zhenguo Li, and
  Chongxuan Li.
\newblock Your absorbing discrete diffusion secretly models the conditional
  distributions of clean data.
\newblock \emph{arXiv preprint arXiv:2406.03736}, 2024.

\bibitem[Sahoo et~al.(2024)Sahoo, Arriola, Schiff, Gokaslan, Marroquin, Chiu,
  Rush, and Kuleshov]{sahoo2024simple}
Subham~Sekhar Sahoo, Marianne Arriola, Yair Schiff, Aaron Gokaslan, Edgar
  Marroquin, Justin~T Chiu, Alexander Rush, and Volodymyr Kuleshov.
\newblock Simple and effective masked diffusion language models.
\newblock \emph{arXiv preprint arXiv:2406.07524}, 2024.

\bibitem[Zhao et~al.(2024{\natexlab{a}})Zhao, Ding, Yu, and
  Akoglu]{zhao2024improving}
Lingxiao Zhao, Xueying Ding, Lijun Yu, and Leman Akoglu.
\newblock Improving and unifying discrete and continuous-time discrete
  denoising diffusion.
\newblock \emph{arXiv preprint arXiv:2402.03701}, 2024{\natexlab{a}}.

\bibitem[Chang et~al.(2022)Chang, Zhang, Jiang, Liu, and
  Freeman]{chang2022maskgit}
Huiwen Chang, Han Zhang, Lu~Jiang, Ce~Liu, and William~T Freeman.
\newblock Maskgit: Masked generative image transformer.
\newblock In \emph{Proceedings of the IEEE/CVF Conference on Computer Vision
  and Pattern Recognition}, 2022.

\bibitem[Heusel et~al.(2017)Heusel, Ramsauer, Unterthiner, Nessler, and
  Hochreiter]{heusel2017gans}
Martin Heusel, Hubert Ramsauer, Thomas Unterthiner, Bernhard Nessler, and Sepp
  Hochreiter.
\newblock {GAN}s trained by a two time-scale update rule converge to a local
  {N}ash equilibrium.
\newblock \emph{Advances in Neural Information Processing Systems}, 30, 2017.

\bibitem[Holtzman et~al.(2019)Holtzman, Buys, Du, Forbes, and
  Choi]{holtzman2019curious}
Ari Holtzman, Jan Buys, Li~Du, Maxwell Forbes, and Yejin Choi.
\newblock The curious case of neural text degeneration.
\newblock In \emph{International Conference on Learning Representations}, 2019.

\bibitem[Ho and Salimans(2022)]{ho2022classifier}
Jonathan Ho and Tim Salimans.
\newblock Classifier-free diffusion guidance.
\newblock \emph{arXiv preprint arXiv:2207.12598}, 2022.

\bibitem[Nisonoff et~al.(2024)Nisonoff, Xiong, Allenspach, and
  Listgarten]{nisonoff2024unlocking}
Hunter Nisonoff, Junhao Xiong, Stephan Allenspach, and Jennifer Listgarten.
\newblock Unlocking guidance for discrete state-space diffusion and flow
  models.
\newblock \emph{arXiv preprint arXiv:2406.01572}, 2024.

\bibitem[Zhao et~al.(2024{\natexlab{b}})Zhao, Shi, Mackey, and
  Linderman]{zhao2024informed}
Yixiu Zhao, Jiaxin Shi, Lester Mackey, and Scott Linderman.
\newblock Informed correctors for discrete diffusion models.
\newblock \emph{arXiv preprint arXiv:2407.21243}, 2024{\natexlab{b}}.

\bibitem[Bradbury et~al.(2018)Bradbury, Frostig, Hawkins, Johnson, Leary,
  Maclaurin, Necula, Paszke, Vander{P}las, Wanderman-{M}ilne, and
  Zhang]{jax2018github}
James Bradbury, Roy Frostig, Peter Hawkins, Matthew~James Johnson, Chris Leary,
  Dougal Maclaurin, George Necula, Adam Paszke, Jake Vander{P}las, Skye
  Wanderman-{M}ilne, and Qiao Zhang.
\newblock {JAX}: composable transformations of {P}ython+{N}um{P}y programs,
  2018.
\newblock URL \url{http://github.com/jax-ml/jax}.

\bibitem[Kitouni et~al.(2023)Kitouni, Nolte, Hensman, and
  Mitra]{kitouni2023disk}
Ouail Kitouni, Niklas Nolte, James Hensman, and Bhaskar Mitra.
\newblock Disk: A diffusion model for structured knowledge.
\newblock \emph{arXiv preprint arXiv:2312.05253}, 2023.

\bibitem[Uria et~al.(2014)Uria, Murray, and Larochelle]{uria2014deep}
Benigno Uria, Iain Murray, and Hugo Larochelle.
\newblock A deep and tractable density estimator.
\newblock In \emph{International Conference on Machine Learning}, pages
  467--475. PMLR, 2014.

\bibitem[Hoogeboom et~al.(2021{\natexlab{b}})Hoogeboom, Gritsenko, Bastings,
  Poole, van~den Berg, and Salimans]{hoogeboom2021autoregressive}
Emiel Hoogeboom, Alexey~A Gritsenko, Jasmijn Bastings, Ben Poole, Rianne
  van~den Berg, and Tim Salimans.
\newblock Autoregressive diffusion models.
\newblock In \emph{International Conference on Learning Representations},
  2021{\natexlab{b}}.

\bibitem[Campbell et~al.(2024)Campbell, Yim, Barzilay, Rainforth, and
  Jaakkola]{campbell2024generative}
Andrew Campbell, Jason Yim, Regina Barzilay, Tom Rainforth, and Tommi Jaakkola.
\newblock Generative flows on discrete state-spaces: Enabling multimodal flows
  with applications to protein co-design.
\newblock In \emph{International Conference on Machine Learning}, 2024.

\bibitem[Santos et~al.(2023)Santos, Fox, Lubbers, and Lin]{santos2023blackout}
Javier~E Santos, Zachary~R Fox, Nicholas Lubbers, and Yen~Ting Lin.
\newblock Blackout diffusion: generative diffusion models in discrete-state
  spaces.
\newblock In \emph{International Conference on Machine Learning}, pages
  9034--9059. PMLR, 2023.

\bibitem[Savinov et~al.(2022)Savinov, Chung, Binkowski, Elsen, and
  Oord]{savinov2021step}
Nikolay Savinov, Junyoung Chung, Mikolaj Binkowski, Erich Elsen, and Aaron
  van~den Oord.
\newblock Step-unrolled denoising autoencoders for text generation.
\newblock In \emph{International Conference on Learning Representations}, 2022.

\bibitem[Shi et~al.(2022)Shi, Zhou, Hwang, Titsias, and
  Mackey]{shi2022gradient}
Jiaxin Shi, Yuhao Zhou, Jessica Hwang, Michalis Titsias, and Lester Mackey.
\newblock Gradient estimation with discrete {S}tein operators.
\newblock In \emph{Advances in Neural Information Processing Systems}, 2022.

\bibitem[Salimans and Knowles(2014)]{salimans2014using}
Tim Salimans and David~A Knowles.
\newblock On using control variates with stochastic approximation for
  variational bayes and its connection to stochastic linear regression.
\newblock \emph{arXiv preprint arXiv:1401.1022}, 2014.

\bibitem[Kool et~al.(2019)Kool, Hoof, and Welling]{Kool2019Buy4R}
W.~Kool, H.~V. Hoof, and M.~Welling.
\newblock Buy 4 {REINFORCE} samples, get a baseline for free!
\newblock In \emph{DeepRLStructPred@ICLR}, 2019.

\bibitem[Mahoney()]{text8}
Matt Mahoney.
\newblock Text8.
\newblock \url{https://mattmahoney.net/dc/textdata.html}.
\newblock Accessed: 2024-05-14.

\bibitem[Gokaslan and Cohen(2019)]{Gokaslan2019OpenWeb}
Aaron Gokaslan and Vanya Cohen.
\newblock Openwebtext corpus.
\newblock \url{http://Skylion007.github.io/OpenWebTextCorpus}, 2019.

\bibitem[Radford et~al.(2019)Radford, Wu, Child, Luan, Amodei, and
  Sutskever]{radford2019language}
Alec Radford, Jeffrey Wu, Rewon Child, David Luan, Dario Amodei, and Ilya
  Sutskever.
\newblock Language models are unsupervised multitask learners.
\newblock \emph{OpenAI blog}, 1\penalty0 (8):\penalty0 9, 2019.

\bibitem[Penedo et~al.(2024)Penedo, Kydl{\'\i}{\v{c}}ek, Lozhkov, Mitchell,
  Raffel, Von~Werra, Wolf, et~al.]{penedo2024fineweb}
Guilherme Penedo, Hynek Kydl{\'\i}{\v{c}}ek, Anton Lozhkov, Margaret Mitchell,
  Colin Raffel, Leandro Von~Werra, Thomas Wolf, et~al.
\newblock The fineweb datasets: Decanting the web for the finest text data at
  scale.
\newblock \emph{arXiv preprint arXiv:2406.17557}, 2024.

\bibitem[Tran et~al.(2019)Tran, Vafa, Agrawal, Dinh, and
  Poole]{tran2019discrete}
Dustin Tran, Keyon Vafa, Kumar Agrawal, Laurent Dinh, and Ben Poole.
\newblock Discrete flows: Invertible generative models of discrete data.
\newblock In \emph{Advances in Neural Information Processing Systems}, 2019.

\bibitem[Zellers et~al.(2019)Zellers, Holtzman, Bisk, Farhadi, and
  Choi]{zellers2019hellaswag}
Rowan Zellers, Ari Holtzman, Yonatan Bisk, Ali Farhadi, and Yejin Choi.
\newblock Hellaswag: Can a machine really finish your sentence?
\newblock \emph{arXiv preprint arXiv:1905.07830}, 2019.

\bibitem[Shih et~al.(2022)Shih, Sadigh, and Ermon]{shih2022training}
Andy Shih, Dorsa Sadigh, and Stefano Ermon.
\newblock Training and inference on any-order autoregressive models the right
  way.
\newblock In \emph{Advances in Neural Information Processing Systems}, 2022.

\bibitem[Ziegler and Rush(2019)]{ziegler2019latent}
Zachary Ziegler and Alexander Rush.
\newblock Latent normalizing flows for discrete sequences.
\newblock In \emph{International Conference on Machine Learning}, 2019.

\bibitem[Van Den~Oord et~al.(2016)Van Den~Oord, Kalchbrenner, and
  Kavukcuoglu]{van2016pixel}
A{\"a}ron Van Den~Oord, Nal Kalchbrenner, and Koray Kavukcuoglu.
\newblock Pixel recurrent neural networks.
\newblock In \emph{International Conference on Machine Learning}, 2016.

\bibitem[Van~den Oord et~al.(2016)Van~den Oord, Kalchbrenner, Espeholt,
  Vinyals, and Graves]{van2016conditional}
Aaron Van~den Oord, Nal Kalchbrenner, Lasse Espeholt, Oriol Vinyals, and Alex
  Graves.
\newblock Conditional image generation with pixelcnn decoders.
\newblock In \emph{Advances in Neural Information Processing systems}, 2016.

\bibitem[Salimans et~al.(2016)Salimans, Karpathy, Chen, and
  Kingma]{salimans2016pixelcnn}
Tim Salimans, Andrej Karpathy, Xi~Chen, and Diederik~P Kingma.
\newblock Pixelcnn++: Improving the pixelcnn with discretized logistic mixture
  likelihood and other modifications.
\newblock In \emph{International Conference on Learning Representations}, 2016.

\bibitem[Chen et~al.(2018)Chen, Mishra, Rohaninejad, and
  Abbeel]{chen2018pixelsnail}
Xi~Chen, Nikhil Mishra, Mostafa Rohaninejad, and Pieter Abbeel.
\newblock Pixelsnail: An improved autoregressive generative model.
\newblock In \emph{International Conference on Machine Learning}, 2018.

\bibitem[Parmar et~al.(2018)Parmar, Vaswani, Uszkoreit, Kaiser, Shazeer, Ku,
  and Tran]{parmar2018image}
Niki Parmar, Ashish Vaswani, Jakob Uszkoreit, Lukasz Kaiser, Noam Shazeer,
  Alexander Ku, and Dustin Tran.
\newblock Image transformer.
\newblock In \emph{International Conference on Machine Learning}, 2018.

\bibitem[Child et~al.(2019)Child, Gray, Radford, and
  Sutskever]{child2019generating}
Rewon Child, Scott Gray, Alec Radford, and Ilya Sutskever.
\newblock Generating long sequences with sparse transformers.
\newblock \emph{arXiv preprint arXiv:1904.10509}, 2019.

\bibitem[Roy et~al.(2021)Roy, Saffar, Vaswani, and Grangier]{roy2021efficient}
Aurko Roy, Mohammad Saffar, Ashish Vaswani, and David Grangier.
\newblock Efficient content-based sparse attention with routing transformers.
\newblock \emph{Transactions of the Association for Computational Linguistics},
  9:\penalty0 53--68, 2021.

\bibitem[Van Den~Oord et~al.(2017)Van Den~Oord, Vinyals, et~al.]{van2017neural}
Aaron Van Den~Oord, Oriol Vinyals, et~al.
\newblock Neural discrete representation learning.
\newblock \emph{Advances in Neural Information Processing Systems}, 30, 2017.

\bibitem[Han et~al.(2024)Han, Kenealy, Barua, Fiedel, and Constant]{Han2024-oe}
Kehang Han, Kathleen Kenealy, Aditya Barua, Noah Fiedel, and Noah Constant.
\newblock Transfer learning for text diffusion models.
\newblock \emph{arXiv preprint arXiv:2401.17181}, 2024.

\bibitem[Bengio et~al.(2015)Bengio, Vinyals, Jaitly, and
  Shazeer]{bengio2015scheduled}
Samy Bengio, Oriol Vinyals, Navdeep Jaitly, and Noam Shazeer.
\newblock Scheduled sampling for sequence prediction with recurrent neural
  networks.
\newblock In \emph{Advances in Neural Information Processing Systems}, 2015.

\bibitem[Glynn(1990)]{Glynn:1990}
Peter~W. Glynn.
\newblock Likelihood ratio gradient estimation for stochastic systems.
\newblock \emph{Communications of the ACM}, 33\penalty0 (10):\penalty0 75--84,
  1990.

\bibitem[Williams(1992)]{williams1992simple}
Ronald~J Williams.
\newblock Simple statistical gradient-following algorithms for connectionist
  reinforcement learning.
\newblock \emph{Machine Learning}, 8\penalty0 (3-4):\penalty0 229--256, 1992.

\bibitem[Peebles and Xie(2023)]{peebles2023scalable}
William Peebles and Saining Xie.
\newblock Scalable diffusion models with transformers.
\newblock In \emph{Proceedings of the IEEE/CVF International Conference on
  Computer Vision}, pages 4195--4205, 2023.

\end{thebibliography}

}

\appendix
\newpage

\begin{table}[ht]
    \footnotesize
    \centering
    \caption{Masking schedule formulas.}
    \label{tab:noise-schedules}
    \renewcommand{\arraystretch}{1.5}
    \begin{tabular}{lcc}
    Masking schedules & $\alpha_t$ & Cross-entropy loss weight $\frac{\alpha'_t}{1 - \alpha_t}$ \\
    \midrule
      Linear   & $1 - t$ & $-\frac{1}{t}$ \\
      Polynomial & $1 - t^w$ & $-\frac{w}{t}$ \\
      Geometric & $\exp\left(-\bar{\beta}_{\text{min}}^{1-t}\bar{\beta}_{\text{max}}^t\right)$ & $-\frac{\exp\left(-\bar{\beta}_{\text{min}}^{1-t}\bar{\beta}_{\text{max}}^t\right)}{1 - \exp\left(-\bar{\beta}_{\text{min}}^{1-t}\bar{\beta}_{\text{max}}^t\right)}\bar{\beta}_{\text{min}}^{1-t}\bar{\beta}_{\text{max}}^t\log \frac{\sigma_{\text{min}}}{\sigma_{\text{max}}}$ \\
      Cosine & $1 - \cos(\frac{\pi}{2}(1 - t))$ & $-\frac{\pi}{2}\tan(\frac{\pi}{2}(1 - t))$ \\
      \bottomrule
    \end{tabular}
\end{table}

\section{Discrete-time derivation}
\label{app:discrete-time-derivation}

We divide time from 0 to 1 into $T$ intervals, and let $s(i) = (i - 1)/T$, $t(i) = i / T$. 
The forward transition matrix $Q_i\in \mathbb{R}^{(m+1)\times (m+1)}$ ($m$ is vocabulary size) at time $t(i)$ is
\begin{align*}
    [Q_i]_{jk} = \begin{cases}
    1 & j = k = m \\
    1 - \beta_i & j = k \neq m \\
    \beta_i & k = m, j \neq m \\
    0 & \text{otherwise}
    \end{cases}
\end{align*}
or more compactly written as
\begin{align*}
    Q_i = (1 - \beta_i)I + \beta_i \mathbf{1} e_m^\top,
\end{align*}
where $\mathbf{1}$ denotes an all-one vector of size $m + 1$, and $e_m$ is an one-hot vector of size $m+1$ with the $m$-th element (recall that counting starts from $0$) being one. 
We use an one-hot vector $x_t$ of length $m+1$ to denote the discrete state. 
The forward conditionals are defined as 
\begin{align} \label{eq:app-per-step-transition}
    q(x_{t(i)}|x_{s(i)}) = \mathrm{Cat}(x_{t(i)}; Q_{i}^\top x_{s(i)}) = x_{s(i)}^\top Q_{i} x_{t(i)},
\end{align}
where $Q_{i}^\top x_{s(i)}$ is the probabilities for each of the $m + 1$ categories that $x_{t(i)}$ can take. 
The marginal forward distribution at time $t(i)$ given $x_0$ is 
\begin{align*}
    q(x_{t(i)}|x_0) = \mathrm{Cat}(x_{t(i)}; \bar{Q}_i^\top x_0) = x_0^\top \bar{Q}_i x_{t(i)},
\end{align*}
where $\bar{Q}_i = \prod_{j=1}^i Q_j = \prod_{j=1}^i (1 - \beta_j)I + \big(1 - \prod_{j=1}^i (1 - \beta_j)\big)\mathbf{1}e_m^\top$. 
To see what this leads to in continuous time, we let $\beta_i = \frac{\beta(t(i))}{T}$ and $T \to \infty$:
\begin{align*}
    \prod_{j=1}^i (1 - \beta_j) &= \exp\Big(\sum_{j=1}^i \log (1 - \beta_j)\Big) \notag \\
    &= \exp\Big(\sum_{j=1}^i - \frac{\beta(t(j))}{T} + o(1/T)\Big) \notag \\
    &\overset{T\to \infty}{\to} \exp\Big(-\int_0^{t(i)} \beta(s) \rmd s\Big). 
\end{align*}
We let $\bar{Q}(t)$ denote the limit of $\bar{Q}_i$ in this case:
\begin{align}
    \bar{Q}(t) &= \exp\big(-\int_0^{t} \beta(s) \rmd s\big) I + \Big(1 - \exp\big(-\int_0^{t} \beta(s) \rmd s\big)\Big) \mathbf{1}e_m^\top \notag \\
    &\triangleq \alpha_t I + (1 - \alpha_t)\mathbf{1}e_m^\top. \notag
\end{align}
Here we define $\alpha_t \triangleq \exp(-\int_0^{t} \beta(s) \rmd s)$. 
And the marginal forward transition is
\begin{align} \label{eq:app_q_xt_given_x0}
    q(x_t|x_0) = \mathrm{Cat}(x_t; \bar{Q}(t)^\top x_0) = x_0^\top \bar{Q}(t) x_t = \alpha_t x_0^\top x_t + (1 - \alpha_t)e_m^\top x_t. 
\end{align}

\section{Continuous-time derivation}\label{app:conttimeformul}

We consider a continuous-time Markov chain with transition rates
\begin{align} \label{eq:Qt}
    Q(t) = (Q_i - I) / (1/T) = \beta(t)(\mathbf{1}e_m^\top - I).
\end{align}
For simplicity, we let $Q = \mathbf{1}e_m^\top - I$. 
The marginal forward distribution at time $t$ given $x_0$ is $q(x_{t}|x_0) = \mathrm{Cat}(x_{t}; \bar{Q}(t)^\top x_0)$, where
\begin{equation*}
    \bar{Q}(t) = \exp\Big( \int_0^t Q(s) \rmd s \Big) = \exp\Big( Q\int_0^t \beta(s) \rmd s\Big) = \exp(\bar{\beta}(t)Q).
\end{equation*}
Here we define $\bar{\beta}(t) \triangleq \int_0^t \beta(s) \rmd s$. 
The matrix exponential can be computed via eigendecomposition:
\begin{align*}
    \bar{\beta}(t)Q = U\Lambda U^{-1},
\end{align*}
where
\begin{align*}
    U &= I - e_m e_m^\top + \frac{1}{\sqrt{n + 1}}\mathbf{1}e_m^\top, \\
    U^{-1} &= I + \sqrt{n + 1}e_m e_m^\top - \mathbf{1}e_m^\top, \\
    \Lambda &= \bar{\beta}(t)(e_m e_m^\top - I),
\end{align*}
and thus $\exp(\Lambda) = \alpha_t I + (1 - \alpha_t)e_m e_m^\top $,
\begin{align*}
    \bar{Q}(t) = U\exp(\Lambda)U^{-1} = \alpha_t I + (1 - \alpha_t) \mathbf{1}e_m^\top. 
\end{align*}
A simpler derivation uses the following property: 
\begin{align*}
    Q^2 = -Q. 
\end{align*}
Therefore, 
\begin{align*}
    \bar{Q}(t) &= \exp(\bar{\beta}(t)Q) \notag \\
    &= I + \bar{\beta}(t)Q + \frac{1}{2}\bar{\beta}(t)^2 Q^2 + \frac{1}{3}\bar{\beta}(t)^3 Q^3 + \dots \notag \\
    &= I + Q - (1 - \bar{\beta}(t) + \frac{1}{2}\bar{\beta}(t)^2 - \frac{1}{3}\bar{\beta}(t)^3 + \dots)Q \notag \\
    &= I + Q - \exp(-\bar{\beta}(t))Q \notag \\
    &= \alpha_t I + (1 - \alpha_t) \mathbf{1}e_m^\top. 
\end{align*}
This marginal forward transition matrix at time $t$ coincides with the result \eqref{eq:bar-Qt-from-discrete} we get by taking the limit of discrete-time derivation.

\paragraph{Arbitrary discretization of the continuous-time forward process. }

For the discrete-time process we have defined the per-step transition in \eqref{eq:app-per-step-transition}. 
For the continuous-time process, we can derive the transition matrix $\bar{Q}(s, t)_{ij} \triangleq q(x_t=j|x_s=i)$ between two arbitrary time $s$ and $t$ as the solution to the following differential equation (known as Kolmogorov forward equation)
\begin{align*}
    \frac{\rmd
    }{\rmd t} \bar{Q}(s, t) &= \bar{Q}(s, t) Q(t) \text{ where } Q(t) = \beta(t) Q
\end{align*}
with initial condition $ \bar{Q}(s, s) = I$. 
The solution is given by
\begin{align*}
    \bar{Q}(s, t) = \exp\big((\bar{\beta}(t) - \bar{\beta}(s))Q\big) = \bar{Q}(s)^{-1}\bar{Q}(t). 
\end{align*}
Routine work (using the Woodbury matrix inversion lemma) shows that
\begin{align*}
    \bar{Q}(t)^{-1} = \alpha_t^{-1} I + (1 - \alpha_t^{-1})\mathbf{1}e_m^\top. 
\end{align*}
Plugging the result back, we get the forward transition distribution from $s$ to $t$:
\begin{align} \label{eq:app_q_xt_given_xs}
    q(x_t|x_s) &= \mathrm{Cat}(x_t; \bar{Q}(s, t)^\top x_s) = x_s^\top \bar{Q}(s, t) x_t, \\
    \text{ where } \bar{Q}(s, t) &\triangleq \bar{Q}(s)^{-1}\bar{Q}(t) = \frac{\alpha_t}{\alpha_s} I + \big(1 - \frac{\alpha_t}{\alpha_s}\big)\mathbf{1}e_m^\top.  \nonumber
\end{align}

\section{Time reversal of the forward process given $x_0$}
\label{app:time-reversal-x0}

The analytic property of our forward process allows to compute many quantities of interest in closed form. 
One such quantity frequently used in diffusion models is the time reversal of the forward process given $x_0$: $q(x_s|x_t, x_0)$. 
We can compute it using \eqref{eq:app_q_xt_given_x0} and \eqref{eq:app_q_xt_given_xs} as
\begin{align} \label{eq:app-q-reverse-simple}
    q(x_s|x_t, x_0) &= \frac{q(x_t|x_s)q(x_s|x_0)}{q(x_t|x_0)} \notag \\
    &= \begin{cases}
        \frac{\alpha_s - \alpha_t}{1 - \alpha_t}x_s^\top x_0 & x_s \neq m, x_t = m  \\
        \frac{1 - \alpha_s}{1 - \alpha_t} & x_s = m, x_t = m \\
        x_s^\top x_t & x_t \neq m. 
    \end{cases} 
\end{align}

Visually, eqn \eqref{eq:app-q-reverse-simple} is a $\mathbb{R}^{(m+1)\times (m+1)}$ matrix (\Cref{fig:reverse-transition-matrix}, left) whose first index is $x_t$ and the second is $x_s$. The matrix is almost an identity matrix except the last row corresponding to $x_t$ is the mask token. The last row means with probability of $\frac{\alpha_s-\alpha_t}{1-\alpha_t}$ the mask token gets unmasked to become $x_0$, and with probability of $\frac{1-\alpha_s}{1-\alpha_t}$ it remains masked. 

Alternatively, we can rewrite the above using reverse transition matrix $\bar{R}^{x_0}(t, s)\in \mathbb{R}^{(m+1)\times (m+1)}$ as
\begin{align*} 
\label{eq:app-q-reverse-compact}
    q(x_s|x_t, x_0) = \mathrm{Cat}(x_s; \bar{R}^{x_0}(t, s)^\top x_t), \text{ where } \bar{R}^{x_0}(t, s) = I + \frac{\alpha_s - \alpha_t}{1 - \alpha_t}e_m(x_0 - e_m)^\top. 
\end{align*}

\begin{figure}[t]
    \centering
    \includegraphics[trim={0 3cm 0 3cm},clip,width=0.92\textwidth]{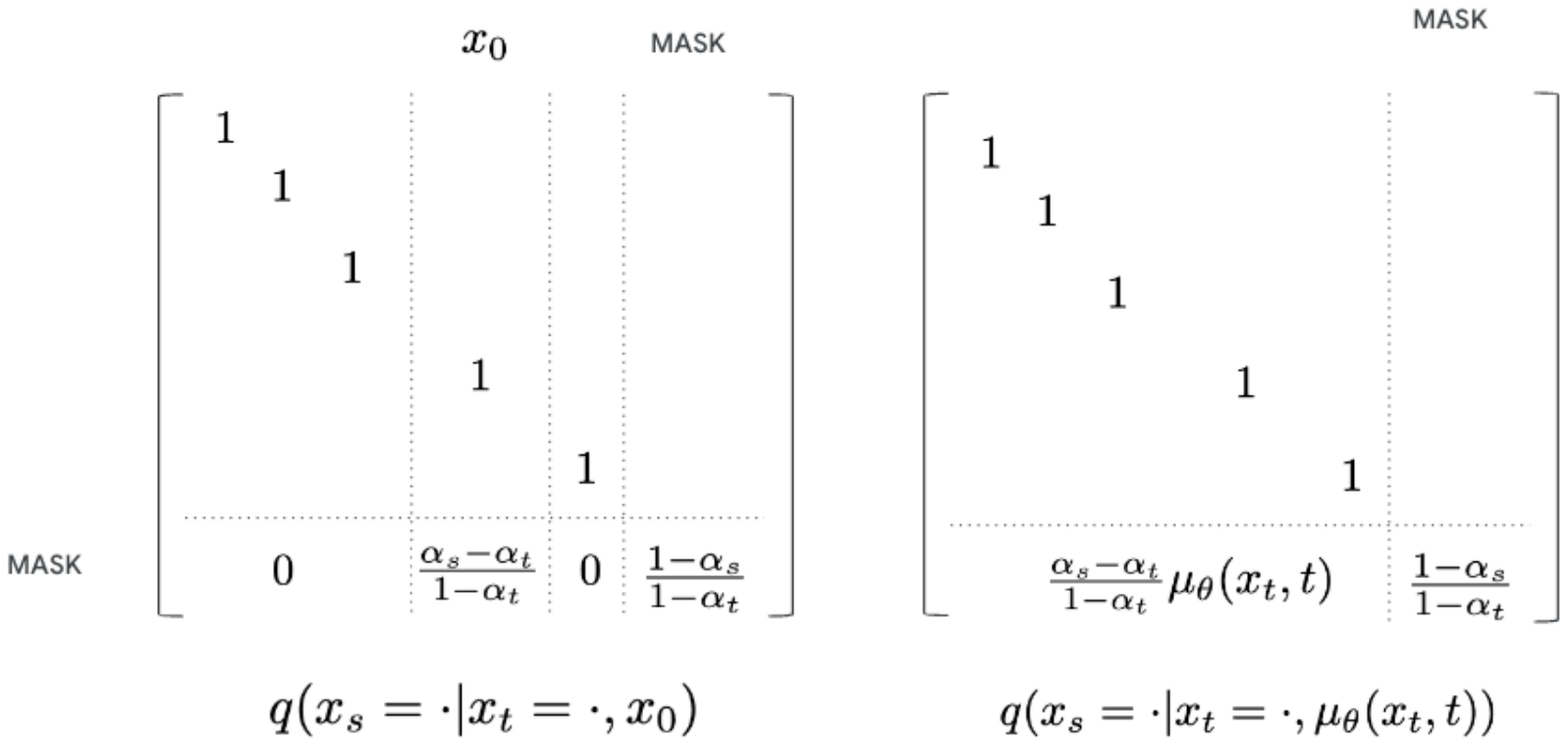}
    \caption{The reverse transition probability and our generative model. Left: $q(x_s = \cdot | x_t = \cdot, x_0)$ in matrix form where first index is $x_t$ and second index is $x_s$. Right: $p_\theta(x_s=\cdot | x_t=\cdot) \triangleq q(x_s=\cdot|x_t=\cdot,\mu_{\theta}(x_t, t))$ also in matrix form.}
    \label{fig:reverse-transition-matrix}
\end{figure}

We are also interested in what would happen in the infinitesimal time limit, i.e., when $s = t - \Delta t$ and $\Delta t \to 0$. Note that %
\begin{align*}
    \alpha_{t - \Delta t} - \alpha_t = -\alpha_t' \Delta t + o(\Delta t). 
\end{align*}
Plugging it into the original formula, we get
\begin{align*}
    \bar{R}^{x_0}(t, t - \Delta t) = I - \frac{\alpha_t'}{1-\alpha_t} e_m(x_0 - e_m)^\top \Delta t + o(\Delta t). 
\end{align*}
Comparing the above with the transition rate matrix $R^{x_0}(t)$ definition
\begin{align*} %
    \bar{R}^{x_0}(t, t - \Delta t) = I + R^{x_0}(t) \Delta t + o(\Delta t),
\end{align*}
we have determined the  transition rate matrix for the reverse process conditioned on $x_0$:
\begin{align} 
    R^{x_0}(t) & = -\frac{\alpha_t'}{1 - \alpha_t} e_m (x_0 - e_m)^\top. 
\end{align}

\section{Details of the ELBO}
\label{app:elbo}
Using \eqref{eq:app-q-reverse-simple} and \eqref{eq:hat-x0}, we compute the KL divergences between forward and reverse transitions 
\begin{align}
    \KL{q(x_s|x_t, x_0)}{p_\theta(x_s|x_t)} 
    &= \KL{q(x_s|x_t, x_0)}{q(x_s|x_t, \mu_\theta(x_t,t))} \label{eq:diff-kl-t} \\
    &= \begin{cases}
    \sum_{x_s=0}^{m} q(x_s|x_t, x_0)\log \frac{q(x_s|x_t, x_0)}{q(x_s|x_t, \mu_\theta(x_t,t))} & x_t = m\\
    0 & x_t \neq m 
    \end{cases} \notag 
    \\
    &= \delta_{x_t=m}\sum_{k\neq m} \frac{\alpha_s - \alpha_t}{1 - \alpha_t} %
    x_0^\top e_k \log \frac{%
    x_0^\top e_k}{%
    \mu_\theta(x_t,t)^\top e_k} \notag \\
    &= -\delta_{x_t=m}\frac{\alpha_s - \alpha_t}{1 - \alpha_t}x_0^\top \log \mu_\theta(x_t,t). \label{eq:app-diff-kl-t2} \notag
\end{align}
Note that $0\log 0 = 0$. 
Alternatively, this result can be easily obtained from the visual depictions of $q(x_s|x_t, x_0)$ and $p_\theta(x_s|x_t)$ shown in \Cref{fig:reverse-transition-matrix}.
In this case, the reconstruction term becomes
\begin{align}
    \E_{q(x_{t(1)}|x_0)}[\log p(x_0|x_{t(1)})] &= \sum_{k=0}^{m} q_{t(1)|0}(k|x_0) \log \frac{q_{t(1)|0}(k|x_0)}{\sum_{j\neq m} q_{t(1)|0}(k|j)} \notag  \\
    &= \alpha_{t(1)} \cdot \log \frac{\alpha_{t(1)}}{\alpha_{t(1)}} + (1 - \alpha_{t(1)})\log \frac{1}{m} \notag \\
    &= - (1 - \alpha_{t(1)}) \log m. 
    \notag
\end{align}

The prior KL term can be computed as
\begin{equation*}
    \KL{q(x_1|x_0)}{p(x_1)} = \KL{\delta_{x_1,m}}{\delta_{x_1,m}} = 0.
\end{equation*}

As usual, we take the continuous-time limit by letting $T \to \infty$:
\begin{align} 
    \mathcal{L}_\infty &\triangleq \lim_{T\to \infty} \mathcal{L}_T \notag \\
    &= \lim_{T\to \infty} \sum_{i=2}^T -\frac{\alpha_{s(i)} - \alpha_{t(i)}}{s(i) - t(i)}\frac{s(i) - t(i)}{1 - \alpha_{t(i)}} x_0^\top \E_{q(x_{t(i)}|x_0)}\left[\delta_{x_{t(i)},m} \log \mu_\theta(x_{t(i)}, t(i))\right] \notag \\
    &= \int_{t(1)}^1 \frac{\alpha_t'}{1 - \alpha_{t}} x_0^\top \E_{q(x_{t}|x_0)}\left[\delta_{x_t,m}\log \mu_\theta(x_{t},t)\right] \rmd t.
    \notag
    \label{eq:app-diff-kl-limit}
\end{align}

\section{Avoiding undefined KL divergence}
\label{app:undefined-kl}

When defining the forward process, we often do not want $\alpha_1$ to be exactly 0, or equivalently, $\lambda_1$ to be $\infty$ for numerical stability reasons. 
Instead, we set $\lambda_1$ to be a finite value, and thereby $\alpha_1$ has a small positive value. 
This has a problem that the support of $q(x_1|x_0)$ is no longer $\{m\}$ and instead becomes $\{m, x_0\}$. 
As a result, the KL divergence between $q(x_1|x_0)$ and $p(x_1)$ is undefined because $q(x_1|x_0)$ is not absolutely continuous with respect to $p(x_1) = \delta_{x_1,m}$.
To resolve the issue, we modify the prior distribution $p(x_1)$ such that it has support over all $m+1$ values. 
One such choice is letting 
\begin{equation*}
    p(x_1) = \frac{\alpha_1}{m} \sum_{j\neq m} \delta_{x_1,j} + (1 - \alpha_1) \delta_{x_1,m}.
\end{equation*}
Then, the prior KL divergence term becomes
\begin{align}
    \KL{q(x_1|x_0)}{p(x_1)} &= \sum_{x_1 = 0}^{m} q(x_1|x_0) \log \frac{q(x_1|x_0)}{p(x_1)} \notag \\
    &= \sum_{x_1=0}^{m} (\alpha_1 \delta_{x_1,x_0} + (1 - \alpha_1) \delta_{x_1,m}) \log \frac{\alpha_1 \delta_{x_1,x_0} + (1 - \alpha_1) \delta_{x_1 = m}}{p(x_1)} \notag \\
    &= \alpha_1 \log \frac{\alpha_1}{\alpha_1 / m} + (1 - \alpha_1) \log \frac{1 - \alpha_1}{1 - \alpha_1} \notag \\
    &= \alpha_1 \log m. \notag
\end{align}

\section{Details of Training and Sampling with MD4}

\subsection{Training}

\begin{algorithm}[h] %
	\caption{A single step of training with MD4.
 }
	\label{alg:training}
	\begin{algorithmic}
 \State {\bf Input:} data minibatch $\{x_t^i\}_{i=1}^B$, network $\mu_{\theta}(\cdot, t)$, masking schedule $\alpha_t$ 
 \For{$i=1,\dots, B$} (in parallel):
 \State $t_i \gets \mathrm{mod}(u + i/B, 1)$, $u \sim U[0, 1]$
 \State for $n \in [N]$, mask out each token $x_0^{i,(n)}$ independently with probability $1-\alpha_{t_i}$ to obtain %
 $x_{t_i}^i$
 \State 
 $\text{for } n \in [N]$, if $x_{t_i}^{(n)} \! = \! m$, compute weighted cross entropy loss $\frac{\alpha_{t_i}'}{1-\alpha_{t_i}} (x_0^{i,(n)})^\top \log \mu^{(n)}_{\theta}(x_{t_i}^i, t_i)$
 \EndFor
 \State Sum over all weighted cross entropy losses for mask positions and optimize via autodiff
	\end{algorithmic}
\end{algorithm}

\subsection{Sampling}

\begin{algorithm}[h] %
	\caption{Unconditional and conditional generation (e.g., infilling) with MD4. 
 }
	\label{alg:sampling}
	\begin{algorithmic}
		\State {\bf Input:} Context sequence $x^c$ of length $N$, with masks indicating the target areas for generation 
            \State {\bf Init:} $\{t(i)\}_{i=0}^T \gets \text{discretize}([0, 1])$, $x_{t(T)} \gets x^c$
		\For{$i=T, T - 1, \dots, 1$}
            \State $t \gets t(i)$, $s \gets t(i-1)$
		\State $\text{for } n \in [N]$, if $x_t^{(n)} = m$, 
  draw $x_s^{(n)} \sim \mathrm{Cat}(\frac{\alpha_s - \alpha_t}{1 - \alpha_t} \mu_\theta^{(n)}(x_t, t) + \frac{1- \alpha_s}{1 - \alpha_t} e_m)$ else $x_s^{(n)} \gets x_t^{(n)}$
		\EndFor
		\State {\bf return} $x_0$.
	\end{algorithmic}
\end{algorithm}

\section{JAX Categorical Sampling and Implicit Top-$p$}
\label{app:jax-implicit-topp}

We noticed that the following equivalent implementation of \Cref{alg:sampling} leads to significantly worse sample quality in JAX:

\begin{algorithm}[h] %
	\caption{Variant of \Cref{alg:sampling} that yields lower sample quality when implemented in JAX.
 }
	\label{alg:two-step}
\begin{algorithmic}
	\State {\bf Input:} Token sequence $x^c$ of length $N$, with masks indicating the target areas for generation
        \State {\bf Init:} $\{t(i)\}_{i=0}^T \gets \text{discretize}([0, 1])$, $x_{t(T)} \gets x^c$
	\For{$i=T, T - 1, \dots, 1$}
        \State $t \gets t(i)$, $s \gets t(i-1)$
	\For{$n \in [N]$} (in parallel)
	\State draw $u \sim U[0, 1]$
	\If{$x_t^{(n)} = m$ \textbf{and} $u < \frac{\alpha_s - \alpha_t}{1 - \alpha_t}$} 
\State draw $x_s^{(n)} \sim \mathrm{Cat}(\mu_\theta^{(n)}(x_t, t))$
    \Else
        \State $x_s^{(n)} \gets x_t^{(n)}$
    \EndIf
	\EndFor
	\EndFor
	\State {\bf return} $x_0$.
\end{algorithmic}
\end{algorithm}

However, mathetically it is equivalent to \Cref{alg:sampling} and should produce identical results. 
Our investigation revealed that the issue arises because  \Cref{alg:sampling} scales the output probabilities of $\mu_\theta$ by a small factor $\frac{\alpha_s - \alpha_t}{1 - \alpha_t}$ as $s$ is close to $t$, causing some categories to have very low probabilities. 
JAX, however, implements categorical sampling using Gumbel argmax, which is less numerically stable than methods like binary search. 
As a result, categories with low probabilities are rarely sampled, even when their cumulative probability is significant. 
In our experiment, we found that categories with probabilities below 1e-8 are rarely sampled out of a total of 50K categories.
Thus, \Cref{alg:sampling} implicitly performs top-$p$ sampling~(with a dynamic p) under JAX’s categorical sampling, yielding better sample quality than \Cref{alg:two-step} where $\mu_\theta$ is not scaled by a small factor and has fewer categories truncated.

\section{Unifying Existing Masked Diffusion Models
\label{app:unified_view}
}

\subsection{%
The CTMC point of view
}
\label{app:ctmc}

We first prove a lemma that connects the forward and reverse transition rate matrices. This follows from the results in \citep{campbell2022continuous} but we give a proof for completeness.

\begin{lemma} \label{lem:forward-reverse-rate}
The forward transition rate matrix $Q(t)$ and the reverse transition rate matrix (given $x_0$) $R^{x_0}(t)$ satisfy:
\begin{align} \label{eq:app_rate-forward-reverse}
    R^{x_0}(t)_{kj} = 
    Q(t)_{jk}\frac{q_{t|0}(j|x_0)}{q_{t|0}(k|x_0)} \text{ for } j \neq k. %
\end{align}
\end{lemma}

\begin{proof}
Consider the reverse transition from time $t+\tau$ to $t$. 
For $j \neq k$, Bayes' rule yields
\begin{align*}
    q(x_t = j|x_{t+\tau}=k, x_0) &= \frac{q(x_t = j|x_0)q(x_{t+\tau} = k|x_t = j)}{q(x_{t+\tau} =k|x_0)} \notag \\
    &= \frac{q(x_t = j|x_0)(\delta_{jk} + Q(t)_{jk}\tau + o(\tau))}{q(x_{t+\tau} = k|x_0)} \notag \\
    &\overset{\tau\to 0}{=} \delta_{kj} + \frac{q(x_t = j|x_0)}{q(x_{t} = k|x_0)}Q(t)_{jk}\tau + o(\tau). 
\end{align*}
Then, it follows from the definition of the transition rate matrix that $R^{x_0}(t)_{kj} = Q(t)_{jk}\frac{q_{t|0}(j|x_0)}{q_{t|0}(k|x_0)}$. 
\end{proof}

\begin{proposition} \label{thm:ctmc}
We use the shorthand $R_\theta(t)_{kj}$ to denote the approximate reverse transition rate from the state $k$ to $j$ obtained by substituting our prediction model $\mu_{\theta}(k)$ for $x_0$ in $R^{x_0}(t)_{kj}$. 
Then, the continuous-time objective \eqref{eq:diff-kl-limit} can be equivalently expressed as
\begin{align} \label{eq:app_kl-limit-rate}
\mathcal{L}_\infty 
= -\int_{t(1)}^1 \E_{q_{t|0}(k|x_0)}\Big[R_\theta(t)_{kk} +\sum_{j\neq k}Q(t)_{kj} \log R_\theta(t)_{jk}\Big] \rmd t + \text{C},
\end{align}
where $C$ is a constant independent of $\theta$. 
\end{proposition}

\begin{proof}%
To rewrite our objective $\mathcal{L}_\infty$ with the transition rate matrices, we first go back to \eqref{eq:diff-kl-t}. 
There, instead of plugging in the explicit form of $\bar{R}^{x_0}(t, s)$, we substitute it with \eqref{eq:rev-rate-x0} which leverages the transition rate $R^{x_0}(t)$. 
To simplify the notation, we assume $x_t = k$ and use the shorthand $R_\theta(t)_{kj} \triangleq R^{\mu_\theta(k)}(t)_{kj}$.  We then have
\begin{align}
    & \KL{q(x_{t-\Delta t}|x_t, x_0)}{p_\theta(x_{t-\Delta t}|x_t)} \notag \\
    &=\KL{\mathrm{Cat}(x_s; \bar{R}^{x_0}(t, t - \Delta t)^\top e_k)}{\mathrm{Cat}(x_s; \bar{R}^{\mu_\theta(k)}(t, t - \Delta t)^\top e_k)} \notag \\
    &= \sum_{j=0}^{m} e_k^\top (I + R^{x_0}(t)\Delta t + o(\Delta t)) e_j \log \frac{e_k^\top (I + R^{x_0}(t)\Delta t + o(\Delta t))e_j}{e_k^\top(I + R_\theta(t)\Delta t + o(\Delta t))e_j} \notag \\
    &= (1 + R^{x_0}(t)_{kk}\Delta t)\log\frac{1 + R^{x_0}(t)_{kk}\Delta t + o(\Delta t)}{1 + R_\theta(t)_{kk}\Delta t + o(\Delta t)} \notag \\
    &\quad + \sum_{j \neq k} (R^{x_0}(t)_{kj}\Delta t)\log \frac{R^{x_0}(t)_{kj}\Delta t + o(\Delta t)}{R_\theta(t)_{kj}\Delta t + o(\Delta t)} + o(\Delta t) \notag \\
    &= (R^{x_0}(t)_{kk} - R_\theta(t)_{kk})\Delta t + \sum_{j \neq k} (R^{x_0}(t)_{kj}\Delta t)\log \frac{R^{x_0}(t)_{kj}\Delta t + o(\Delta t)}{R_\theta(t)_{kj}\Delta t + o(\Delta t)} + o(\Delta t). \notag
\end{align}
For the last identity, we have used the fact that $\log (1 + x) = x + o(x)$. 
To obtain $\mathcal{L}_\infty$, we take the limit of $\mathcal{L}_T$ as $T\to \infty$, which is equivalent to letting $\Delta t = 1/T \to 0$. We obtain
\begin{align}
    \mathcal{L}_{\infty} &= \lim_{T \to \infty} \sum_{i=2}^T \E_{q(x_{t(i)}|x_0)} [\KL{q(x_{s(i)}|x_{t(i)}, x_0)}{p_\theta(x_{s(i)}|x_{t(i)})}] \notag \\
    &= \lim_{T \to \infty} \sum_{i=2}^T \E_{q(x_{t(i)}|x_0)}\Big[\Big(R^{x_0}(t(i))_{kk} - R_\theta(t(i))_{kk}  \notag \\
    &\quad + \sum_{j\neq k}R^{x_0}(t(i))_{kj} \log \frac{R^{x_0}(t(i))_{kj}\Delta t + o(\Delta t)}{R_\theta(t(i))_{kj}\Delta t + o(\Delta t)}\Big) \Delta t + o(\Delta t)\Big] \notag \\
    &= \int_{t(1)}^1 \E_{q_{t|0}(k|x_0)}\Big[R^{x_0}(t)_{kk} - R_\theta(t)_{kk} + \sum_{j\neq k}R^{x_0}(t)_{kj} \log \frac{R^{x_0}(t)_{kj}}{R_\theta(t)_{kj}}\Big] \rmd t. \notag 
\end{align}
Note that $R^{x_0}(t)$ is a constant matrix independent of $\theta$. 
Absorbing all constant terms into $C$, we have
\begin{equation*}
    \mathcal{L}_{\infty} = -\int_{t(1)}^1 \E_{q_{t|0}(k|x_0)}\Big[R_\theta(t)_{kk} + \sum_{j\neq k}R^{x_0}(t)_{kj} \log R_\theta(t)_{kj}\Big] \rmd t + C.
\end{equation*}
Next, we subtitute $R^{x_0}(t)$ with the forward transition rate using \Cref{lem:forward-reverse-rate}:
\begin{align*}
    \mathcal{L}_{\infty} &= -\int_{t(1)}^1 \E_{q_{t|0}(k|x_0)}\Big[R_\theta(t)_{kk} + \sum_{j\neq k}Q(t)_{jk}\frac{q_{t|0}(j|x_0)}{q_{t|0}(k|x_0)} \log R_\theta(t)_{kj}\Big] \rmd t + C \\
    &= -\int_{t(1)}^1 \Big[\sum_{k=0}^m q_{t|0}(k|x_0) R_\theta(t)_{kk} + \sum_{k=0}^m \sum_{j\neq k}Q(t)_{jk}q_{t|0}(j|x_0) \log R_\theta(t)_{kj}\Big] \rmd t + C \\
    &= -\int_{t(1)}^1 \Big[\sum_{k=0}^m q_{t|0}(k|x_0) R_\theta(t)_{kk} + \sum_{k=0}^m \sum_{j\neq k}Q(t)_{kj}q_{t|0}(k|x_0) \log R_\theta(t)_{jk}\Big] \rmd t + C,
\end{align*}
where the last identity used the discrete analog to integration-by-part (or summation-by-part): $\sum_{k=0}\sum_{j\neq k} f(j, k) = \sum_{k=0}\sum_{j\neq k} f(k, j)$. 
Rearranging the terms then gives \eqref{eq:app_kl-limit-rate}. 

\end{proof}

\subsection{Differences from \citet{campbell2022continuous}}
\label{app:campbell-diff}

\citet{campbell2022continuous} used the first term of \eqref{eq:app_kl-limit-rate} as the training loss. 
A key limitation of this loss function is from the inner summation term
\begin{align*}
    \sum_{j\neq k} Q(t)_{kj} \log R_\theta(t)_{jk}.
\end{align*}
For single dimension case, the sum is analytically computable due to the sparse structure of $R_\theta(t)$---if $x_t = k$ is mask, the second term disappears; otherwise the only possible neighbor $j$ is a mask. 
However, for multidimensional data, $j$ will represent all $N-1$ neighbors in the forward process, i.e., the states we get from mask out a single unmasked dimension of $x_t = k$.
Recall that $R_\theta(t)_{jk}$ is computed as substituting our neural network prediction model $\mu_\theta(j)$ for $x_0$ in $R^{x_0}(t)_{jk}$. 
Therefore, the summation together with $R_\theta(t)_{kk}$ requires $N$ evaluations of $\mu_\theta(\cdot)$. 
This is prohibitive since the neural network model is usually expensive. 
To resolve this issue, \citet{campbell2022continuous} proposed to rewrite the sum as
\begin{align*}
    \E_{j \sim \tilde{q}(\cdot|k)}\left[Z_k \log R_\theta(t)_{jk}\right] \text{\quad where\quad} \tilde{q}(j|k) = \frac{Q(t)_{kj}}{Z_k}, Z_k \triangleq \sum_{j'\neq k} Q(t)_{kj'}
\end{align*}
and estimate it through Monte Carlo. 
Taking into account the outer expectation under $q_{t|0}(k|x_0)$, the computation of the loss then becomes a doubly stochastic estimate (using $k\sim q_{t|0}(k|x_0)$ and $j\sim \tilde{q}(j|k)$) which suffers from large variance. 
In contrast, the form of our loss \eqref{eq:diff-kl-limit} only requires evaluating $\mu_\theta$ once for a single stochastic estimation of the expectation w.r.t. $q(x_t|x_0)$.

\subsection{Score parameterization}
\label{app:score}

We provide a simpler derivation of the score-based loss \citep{benton2024denoising,lou2023discrete} below.
We start from the form of the ELBO in \eqref{eq:app_kl-limit-rate} and rewrite it as
\begin{align}
    \mathcal{L}_{\infty} &= \int_{t(1)}^1 \E_{q_{t|0}(k|x_0)}\Big[\sum_{j\neq k}\Big(R^{\mu_\theta}(t)_{kj} - R^{x_0}(t)_{kj} + R^{x_0}(t)_{kj} \log \frac{R^{x_0}(t)_{kj}}{R^{\mu_\theta}(t)_{kj}}\Big)\Big] \rmd t. 
    \label{eq:app_diff-kl-cont-in-rate}
\end{align}
For the last identity we used the zero-row-sum property of transition rate matrix:
\begin{equation*}
    R^{x_0}(t)_{kk} = -\sum_{j\neq k} R^{x_0}(t)_{kj}. 
\end{equation*}
If we plug \eqref{eq:app_rate-forward-reverse} into \eqref{eq:app_diff-kl-cont-in-rate} and reparameterize with a score model 
\begin{align} \label{eq:app_score-model-q}
    s_\theta(x_t)_j \triangleq \frac{q_{t|0}(j|\mu_\theta(x_t))}{q(x_t|\mu_\theta(x_t))},
\end{align}
we recover the score entropy loss function from \citet{benton2024denoising,lou2023discrete}:
\begin{equation} \label{eq:app_score-entropy}
    \mathcal{L}_{\infty} = \int_{t(1)}^1 \E_{q_{t|0}(k|x_0)}\Big[\sum_{j\neq k} Q(t)_{jk}\Big(s_\theta(k)_j - \frac{q_{t|0}(j|x_0)}{q_{t|0}(k|x_0)}\log s_\theta(k)_j + \psi\Big(\frac{q_{t|0}(j|x_0)}{q_{t|0}(k|x_0)}\Big)\Big)\Big] \rmd t, 
\end{equation}
where $\psi(y) \triangleq y\log y - y$. 
Note that our derivation above is different and simpler than that of \citet{campbell2022continuous} (which \citet{lou2023discrete} is based on) since we leverage the conditional reverse transition rate given $x_0$ instead of the transition rate matrix of the reverse process. 
We can further simplify the loss with the following relationship between the conditional score and $x_0$:
\begin{align}\label{eq:app_score-to-x0}
    \frac{q_{t|0}(j|x_0)}{q_{t|0}(k|x_0)} = \frac{x_0^\top \bar{Q}(t) e_j}{x_0^\top \bar{Q}(t) e_k} 
    = \frac{\alpha_t}{1 - \alpha_t}x_0^\top e_j & \text{ for }  k = m, j\neq k. 
\end{align}
Note that only the result under the case $k=m$ is needed.
This is because when $x_t$ is unmasked, at any time between $0$ and $t$, the state must stay unchanged and remain $x_0$. 
As a result, $\KL{q(x_{t-\Delta t}|x_t, x_0)}{p_\theta(x_{t-\Delta t}|x_t)} = 0$ for $x_t \neq m$. 
From \eqref{eq:Qt}, we know $Q(t)_{jk} = \beta(t)(\delta_{mk} - \delta_{jk})$. 
Combining  \eqref{eq:app_score-to-x0} and \eqref{eq:app_score-entropy}, we get
\begin{align} \label{eq:app_masked-kl-inf-score}
    \mathcal{L}_\infty = \int_{t(1)}^1 \beta(t) \Big( \E_{q_{t|0}(k|x_0)}\big[\delta_{mk} \big(\sum_{j\neq k} s_\theta(k)_j - \frac{\alpha_t}{1 - \alpha_t} x_0^\top\log s_\theta(k)\big)\big] + \psi\big(\frac{\alpha_t}{1 - \alpha_t}\big)\Big) \rmd t.
\end{align}
Further, we can show the connection between \eqref{eq:app_masked-kl-inf-score} and \eqref{eq:diff-kl-limit} by reverting the score parameterization to a mean parameterization using \eqref{eq:app_score-model-q}, or equivalently
$
    s_\theta(x_t)_j = \frac{\alpha_t}{1 - \alpha_t} \mu_\theta(x_t)^\top e_j
$. 
By doing so, we obtain
\begin{align}
    \mathcal{L}_\infty = \int_{t(1)}^1 \beta(t) \Big( \E_{q_{t|0}(k|x_0)}\big[\delta_{mk} \big(\sum_{j\neq k}s_\theta(k)_j - \frac{\alpha_t}{1 - \alpha_t} x_0^\top\log \mu_\theta(k) \big] + \frac{\alpha_t}{1 - \alpha_t}\big) \rmd t. \notag
\end{align}
Observing that
\begin{align} \label{eq:app_score-constraint}
    \sum_{j\neq m} s_\theta(m)_j = \frac{\alpha_t}{1 - \alpha_t}, 
\end{align}
we conclude that this recovers the objective in \eqref{eq:diff-kl-limit}. 
Interestingly, in \citet{lou2023discrete} the score parameterization is not constrained to satisfy \eqref{eq:app_score-constraint}. 
That means the learned reverse model might be incompatible with the forward process. 

Below, we prove \Cref{prop:score-mean-relation} using the result from \Cref{eq:app_score-to-x0}.

\begin{proofof}{\Cref{prop:score-mean-relation}}
    \begin{align*}
        \frac{q_{t}(j)}{q_{t}(m)} &= \frac{\sum_{x_0}q_{t|0}(j|x_0)q(x_0)}{q_t(m)} 
        = \frac{\sum_{x_0}q_{t|0}(j|x_0)q_{0|t}(x_0|m)}{q_{t|0}(m|x_0)} 
        = \E_{x_0|x_t=m}\left[\frac{q_{t|0}(j|x_0)}{q_{t|0}(m|x_0)}\right] \\
        &= \E_{x_0|x_t=m}\left[\frac{\alpha_t}{1 - \alpha_t}x_0^\top e_j\right] 
        = \frac{\alpha_t}{1 - \alpha_t}\E[x_0|x_t=m]^\top e_j.
    \end{align*}
\end{proofof}

\subsection{Other related work.} \label{app:other-related}

\paragraph{MaskGIT~\citep{chang2022maskgit}.} MaskGIT is a diffusion-inspired iterative denoising model for discrete image tokens obtained through models such as VQ-VAE~\citep{van2017neural}. 
Training of  MaskGIT follows the steps: (a) Sample $t \in [0, 1]$. (b) Given a mask scheduling function $\gamma(t)$, %
sample $\gamma(t) N$ tokens to place masks. (c) For data $x_0$ of size $(m+1)\times \numtokens$  and the partially masked state $x_t$, minimize the negative log-likelihood
    \begin{align}\label{eq:Maskgitloss}
        \mathcal{L}_{\text{MaskGIT}} = -\int^1_0
        \E_{x_t}\Big[\textstyle \sum_{\indextoken: x_{t}^{(\indextoken)}= m} (x_{0}^{(\indextoken)})^\top \log \mu_\theta^{(\indextoken)}(x_t,t)\Big]\rmd t.
    \end{align}
Our forward process satisfies
$
q_{t|0}(m|x_0) = 1 - \alpha_t
$.
Therefore, when we set the mask scheduling function as
$\gamma(t) = 1 - \alpha_t$ 
we obtain a loss similar to \eqref{eq:multi-diff-kl-limit} without the $\frac{\alpha_t'}{1 - \alpha_t}$ weighting. 
Note that there remains a difference in the sampling distribution of $x_t$: in the masked diffusion forward process, tokens are sampled independently and do not necessarily yield exactly $(1 - \alpha_t)N$ mask tokens at time $t$, though the expected number is $(1 - \alpha_t)N$.
One might be interested in whether the uniform weighting can be recovered by selecting an appropriate schedule $\alpha_t$. However, solving $\alpha_t$ such that 
$\alpha_t' = \alpha_t - 1$ yields $\alpha_t = c e^t + 1$ and there is no $c$ that satisfies both $\alpha_0 = 1$ and $\alpha_1 = 0$. 
This shows that training with the MaskGIT loss \eqref{eq:Maskgitloss} may not be faithfully optimizing the model likelihood.   

\paragraph{Discrete flow matching~\citep{campbell2024generative}.} For the linear schedule $\alpha_t = 1 - t$, our reverse transition rate matrix~\eqref{eq:rev-rate-x0} conditioned on $x_0$ is:
\begin{align*}
    R^{x_0}(t) = -\frac{\alpha'_t}{1 - \alpha_t} e_m (x_0 - e_m)^\top = \frac{1}{t} e_m(x_0 - e_m)^\top.
\end{align*}
This is the same as the conditional reverse  transition rate used in \citet[Eq. (22)]{campbell2024generative}---note that their time $t$ is reversed, and the rate matrix was therefore in the form $R^{x_0}(t) = \frac{1}{1 - t} e_m(x_0 - e_m)^\top$. 

\paragraph{SDDM~\citep{sunscore}.} 
\citet{sunscore} proposed a pseudo-likelihood-like objective for training discrete diffusion models that can also be applied to masked diffusion. 
However, their objective encounters the same challenge as \citet{campbell2022continuous} --- requiring $N$ passes of the mask prediction model. 
To mitigate this, they introduced a new transformer architecture, which unfortunately leads to some performance degradation.

\paragraph{Blackout diffusion~\citep{santos2023blackout}.}
\citet{santos2023blackout} proposed a ``blackout'' diffusion process that gradually diffuses images to a black state. While this approach is similar to masked diffusion on binary data, key differences emerge when dealing with larger state spaces. In their method, image pixel intensities gradually fade out, whereas ours directly transition to a mask state. Our method offers more flexibility, being applicable to general discrete state spaces without requiring predefined structural relationships. It also demonstrates competitive performance in image generation, achieving this without knowing pixel value proximity.

\paragraph{SUNDAE~\citep{savinov2021step, Han2024-oe}.} 
Unlike masked diffusion, SUNDAE %
uniformly corrupts data with random tokens in the vocab (known as uniform discrete diffusion~\citep{austin2021structured}). 
Additionally, it uses a second loss term from cross entropy between clean data and 1-step unrolled model prediction. Similar ideas have been proposed in \citep{bengio2015scheduled}.

\section{Details for state-dependent rates}
\label{app:state-dependent}

\subsection{Derivations and time continuous limit} \label{app:state-dependentfirst}

All derivations in this section assume that $x_t$ 
is a single token, while 
for $\numtokens$ tokens the masked diffusion with state-dependent rates factorises across 
the $\numtokens$ tokens. 
Learning from data of $\numtokens$ 
tokens using variational inference is discussed in  \Cref{app:grad-est}. 

Given the forward transition $q(x_t | x_s)$ 
and marginal $q(x_s | x_0)$ derived in main text
(\Cref{sec:dependent_rates})
The reversal given $x_0$ is $q(x_s|x_t, x_0) = \mathrm{Cat}(x_s; \bar{R}^{x_0}(t, s)^\top x_t)$ for
\begin{align*}
    \bar{R}^{x_0}(t, s)_{jk} = \begin{cases}
        \big(\frac{\alpha_s - \alpha_t}{\mathbf{1} - \alpha_t}\big)^\top x_0 x_0^\top e_k & j = m, k \neq m \\
        \big(\frac{\mathbf{1} - \alpha_s}{\mathbf{1} - \alpha_t}\big)^\top x_0 & j = m, k = m \\
        \delta_{jk} & j \neq m. 
    \end{cases}
\end{align*}
or alternatively can be written as 
\begin{align}
 q(x_s | x_t, x_0) 
& = \frac{q(x_t | x_s) 
q(x_s | x_0)}
{q(x_t | x_0)} 
\nonumber \\
& = \frac{
\left[
\frac{\alpha_t^\top x_s}{\alpha_s^\top x_s} x_s^\top x_t + (1 -\frac{\alpha_t^\top x_s}{\alpha_s^\top x_s}) e_m^\top x_t
\right] 
\left[ \alpha_s^\top x_0 x_0^\top x_s + (1 - \alpha_s^\top x_0) e_m^\top x_s 
\right]}
{\left[ \alpha_t^\top x_0 x_0^\top x_t + (1 - \alpha_t^\top x_0) e_m^\top x_t 
\right]}.
\label{eq:app_qxsxtx0}
\end{align}
To simplify this expression we consider the two cases: either $x_t = m$ (i.e.\ $x_t$ is mask) or $x_t \neq m$ where in the second case $x_t = x_0$. For the case $x_t =m$, the denominator in \eqref{eq:app_qxsxtx0} 
simplifies as 
$$
q(x_t=m|x_0) = 1 - \alpha_t^\top x_0
$$ 
due to $x_0^\top x_t =0$ 
since $x_0 \neq m $, i.e.\ the observed token $x_0$ cannot be a mask.  
Then given that $x_t = m$  the probability that $x_s = x_t = m$ is
\begin{equation}
\frac{1- \alpha_s^\top x_0}
{1  -  \alpha_t^\top x_0}
= 
\frac{({\bf 1} - \alpha_s)^\top x_0}
{ ({\bf 1} -  \alpha_t)^\top x_0}
= \left(\frac{{\bf 1} - \alpha_s
}
{ {\bf 1} -  \alpha_t}\right)^\top x_0
\label{eq:app_qxsxt_em_x01}
\end{equation}
while the remaining probability  for $x_s = x_0 \neq m$ is
\begin{equation}
\frac{(\alpha_s - \alpha_t)^\top x_0}
{1 -  \alpha_t^\top x_0}
= 
\frac{(\alpha_s - \alpha_t)^\top x_0}
{ ({\bf 1} -  \alpha_t)^\top x_0}
= \left(\frac{\alpha_s - \alpha_t}
{ {\bf 1} -  \alpha_t}\right)^\top x_0.
\label{eq:app_qxsxt_em_x02}
\end{equation}
Then, combining \eqref{eq:app_qxsxt_em_x01} and \eqref{eq:app_qxsxt_em_x02} to write $q(x_s | x_t=m, x_0)$ in an unified way yields the expression 
\eqref{eq:xsxt_em_x_0} 
in the main \Cref{sec:dependent_rates}. In the second case, 
when $x_t = x_0 \neq m$, 
$q(x_s | x_t \neq m, x_0 )$ from \eqref{eq:app_qxsxtx0} simplifies dramatically 
and it becomes $q(x_s | x_t \neq m, x_0 ) = x_t^\top x_s$ which is a point mass that sets $x_s=x_t$. 

\paragraph{Derivation of the continuous-time limit of the loss in 
\eqref{eq:vec-diff-kl-limit}.}
To simplify the notation, we let $\xi_{s,t} \triangleq \frac{\alpha_s - \alpha_t}{1 - \alpha_t}$. 
We first compute the KL divergence terms in the discrete-time ELBO as
\begin{align*}
    &\KL{q(x_s|x_t, x_0)}{p_\theta(x_s|x_t)} \notag \\
    &= \begin{cases}
    \sum_{x_s=0}^{m} q(x_s|x_t, x_0)\log \frac{q(x_s|x_t, x_0)}{p_\theta(x_s|x_t)} & x_t = m\\
    0 & x_t \neq m 
    \end{cases} \notag 
    \\ 
    &= \delta_{x_t,m}\Big[\sum_{k\neq m} \xi_{s,t}^\top x_0 x_0^\top e_k 
    \log \frac{\xi_{s,t}^\top x_0 x_0^\top e_k}{\xi_{s,t}^\top \diag(\mu_\theta(x_t,t)) e_k} + (1 - \xi_{s,t})^\top x_0 \log \frac{(1 - \xi_{s,t})^\top x_0}{(1 - \xi_{s,t})^\top \mu_\theta(x_t,t)}\Big] \notag \\
    &= \delta_{x_t,m}\Big[ -\xi_{s,t}^\top x_0 x_0^\top \log \mu_\theta(x_t,t) + (1 - \xi_{s,t})^\top x_0 \log \frac{(1 - \xi_{s,t})^\top x_0}{(1 - \xi_{s,t})^\top \mu_\theta(x_t,t)}\Big].
\end{align*}
Let $\Delta_t \triangleq \frac{1}{T} = t(i) - s(i)$ for all $i$. 
Plugging $\alpha_{t - \Delta t} = \alpha_t - \alpha_t'\Delta t + o(\Delta t)$ into the above formula and letting $\gamma_t = \frac{\alpha_t'}{1 - \alpha_t}$, we get
\begin{align*}
    &\KL{q(x_s|x_t, x_0)}{p_\theta(x_s|x_t)} \notag \\
    &= 
    \delta_{x_t,m} \left[ \gamma_t^\top x_0 x_0^\top \log \mu_\theta(x_t,t) \Delta t + \left(1 + \gamma_t^\top x_0 \Delta t\right) \cdot \log \frac{1 + \gamma_t^\top x_0 \Delta t + o(\Delta t)}{1 + \gamma_t^\top \mu_\theta(x_t,t) \Delta t + o(\Delta t)} + o(\Delta t)\right] \\
    &= 
    \delta_{x_t,m} \left[ \gamma_t^\top x_0 x_0^\top \log \mu_\theta(x_t,t) \Delta t + \left(1 + \gamma_t^\top x_0 \Delta t\right) \left(\gamma_t^\top x_0 \Delta t - \gamma_t^\top \mu_\theta(x_t,t) \Delta t + o(\Delta t) \right) + o(\Delta t)\right] \\
    &= 
    \delta_{x_t,m} \left[ \gamma_t^\top x_0 x_0^\top \log \mu_\theta(x_t,t) \Delta t + \gamma_t^\top x_0 \Delta t - \gamma_t^\top \mu_\theta(x_t,t) \Delta t  + o(\Delta t)\right] \\
    &= 
    \delta_{x_t,m}\cdot \gamma_t^\top (x_0 x_0^\top \log \mu_\theta(x_t,t) + x_0  - \mu_\theta(x_t,t))\Delta t   + o(\Delta t).
\end{align*}
Therefore,
\begin{align*}
    &\lim_{T \to \infty} \;\sum_{i=2}^T \E_{q(x_{t(i)}|x_0)}[\KL{q(x_{s(i)}|x_{t(i)}, x_0)}{p_\theta(x_{s(i)}|x_{t(i)})}] \\
    &=\lim_{T \to \infty} \;\sum_{i=2}^T \E_{q(x_{t(i)}|x_0)}[\delta_{x_{t(i)},m}\cdot \gamma_t^\top (x_0 x_0^\top \log \mu_\theta(x_{t(i)}, t(i)) + x_0  - \mu_\theta(x_{t(i)}, t(i)))\Delta t   + o(\Delta t)] \\
    &= \int_{t(1)}^1 \gamma_t^\top \E_{q(x_{t(i)}|x_0)}[\delta_{x_t,m}\cdot  (x_0 x_0^\top \log \mu_\theta(x_t,t) + x_0  - \mu_\theta(x_t,t))] \rmd t.
\end{align*}
Letting $t(1) \to 0$ proves the result.

\subsection{Training and gradient estimation}
\label{app:grad-est}

The model is applied 
to data consisted of $\numtokens$
tokens where $x_0 = (x_0^{1}, \ldots, x_0^{(\numtokens)})$ and where each state in the masked diffusion is 
 $x_t = (x_t^{1}, \ldots, x_t^{(\numtokens)})$. The  reverse generated 
 model has a factorizing transition conditional of the form 
 $\prod_{\indextoken=1}^\numtokens p_\theta(x_s^{(\indextoken)} | x_t)$
where 
$p_\theta(x_s^{(\indextoken)} | x_t) = q(x_s^{(\indextoken)} | x_t^{(\indextoken)}, \mu_\theta^{(\indextoken)}(x_t,t))$  has a form that depends on whether  $x_t^{(\indextoken)}=m$ or
$x_t^{(\indextoken)} \neq m$. For the first case:  
$$
p_\theta(x_s^{(\indextoken)} | x_t^{(\indextoken)} = m, \{x_t^{(k)}\}_{k \neq \indextoken} ) 
= \Big(\frac{{\bf 1} - \alpha_s
}
{ {\bf 1} -  \alpha_t}\Big)^\top \mu^{(\indextoken)}_\theta(x_t,t) e_m^\top x_s^{(\indextoken)} 
+ \Big(\frac{\alpha_s - \alpha_t}
{ {\bf 1} -  \alpha_t}\Big)^\top \diag(\mu^{(\indextoken)}_\theta(x_t,t))  x_s^{(\indextoken)}, 
$$
where $\mu^{(\indextoken)}_\theta(x_t,t) = \text{softmax}(f_\theta(x_t))$ is a $m+1$ dimensional probability vector modelled by a NN (where the final value is constrained to be zero since  $\mu^{(\indextoken)}_\theta(x_t,t)$ is a reconstruction of  
$x_0^{(\indextoken)}$ which
cannot be mask, so in practice the NN classifier needs to have a softmax output only over the $m$ actual token classes). Crucially, note that the NN classifier receives as input the full
state $x_t$ of all tokens, while additional time features to encode $t$ are also included. When 
$x_t^{(\indextoken)} \neq m$ 
the reverse transition 
model is set to be 
$
p_\theta(x_s | x_t^{(\indextoken)} \neq m, \{x_t^{(k)}\}_{k \neq \indextoken} ) = (x_t^{(\indextoken)})^\top x_s^{(\indextoken)}$ which matches precisely 
$q(x_s^{(\indextoken)} | x_t^{(\indextoken)}=m, x_0^{(\indextoken)}) = (x_t^{(\indextoken)})^\top x_s^{(\indextoken)}$ from the forward process.   

The full negative lower bound for state-dependent rates and
assuming $\numtokens$ tokens is given by  
\begin{align} \label{eq:genmd4-n}
\mathcal{L}^{(N)}_\infty 
    = \int_0^1 \Big(\frac{\alpha_t'}{1 - \alpha_{t}}\Big)^\top \E_{q(x_{t}|x_0)}\Big[\textsum_{\indextoken:x_t^{(\indextoken)}=m }  (x_0^{(\indextoken)} - \mu_\theta^{(\indextoken)}(x_t,t) + x_0^{(\indextoken)} (x_0^{(\indextoken)})^\top \log \mu_\theta^{(\indextoken)}(x_t,t))\Big] \rmd t. 
\end{align} 
Given that each $\alpha_{t,i} = 1 - t^{w_i}$, the reverse model becomes 
$$
p_\theta(x_s^{(\indextoken)} | x_t^{(\indextoken)} \neq m, \{x_t^{(k)}\}_{k \neq \indextoken} ) 
=   \left(e^{ w \log \frac{s}{t}}\right)^\top   \mu_\theta^{(\indextoken)}(x_t,t) e_m^\top x_s^{(\indextoken)} 
+ \left(1 - e^{ w \log \frac{s}{t}}\right)^\top \text{diag}(\mu_\theta^{(\indextoken)}(x_t,t))  x_s^{(\indextoken)}, 
$$
where $w$ is the $m+1$ dimensional vector of all $w_i$s. Note that the probability of $x_s^{(\indextoken)}$ staying in the mask state, i.e., $x_s^{(\indextoken)} = m$  depends on the full $x_t$ and it is given by 
$
\left(e^{ w \log \frac{s}{t}}\right)^\top \mu_\theta^{(\indextoken)}(x_t,t)
= \sum_{i=0}^{m-1} e^{ w_i \log \frac{s}{t}} \mu_\theta^{(\indextoken)}(x_t,t)_i
$
while the probability for $x_s^{(\indextoken)}$ 
to take a certain non-mask token value $i$ is 
$
\left(1 - e^{ w_i \log \frac{s}{t}}\right)
\mu_\theta^{(\indextoken)}(x_t,t)_i. 
$
The gradient wrt $t$ is $\alpha_{t,i}' = - w_i t^{w_i -1}$ and 
$\frac{\alpha_{t,i}'}{1 - \alpha_{t,i}} = - \frac{w_i}{t}$ the above loss is written as 
\begin{align*}    
\mathcal{L}^{(N)}_\infty 
    = - \int_0^1 \frac{1}{t} w^\top \E_{q(x_{t}|x_0)}\left[\textsum_{\indextoken:x_t^{(\indextoken)}=m }  (x_0^{(\indextoken)} - \mu_\theta^{(\indextoken)}(x_t,t) + x_0^{(\indextoken)} (x_0^{(\indextoken)})^\top \log \mu_\theta^{(\indextoken)}(x_t,t))\right] \rmd t,
\end{align*} 
where $w$ is the vector of all $w_i$'s. 
An unbiased gradient over the NN parameters $\theta$ is straightforward 
to obtain since we just need to sample one time point $t$ and an $x_t \sim q(x_t|x_0)$ to approximate the integral and  expectation and then
use the gradient: 
$$
- \nabla_\theta  \sum_{\indextoken:x_t^{(\indextoken)}=m } 
\frac{1}{t}
w^\top
\left(x_0^{(\indextoken)} - \mu_\theta^{(\indextoken)}(x_t,t) + x_0^{(\indextoken)} (x_0^{(\indextoken)})^\top \log \mu_\theta^{(\indextoken)}(x_t,t) \right).
$$
The gradient wrt the  $w$ parameters is more complex since   
these parameters appear also in the discrete 
distribution $q(x_t | x_0)$ which is not reparametrizable. To deal with this we need REINFORCE unbiased gradients~~\citep{Glynn:1990,williams1992simple}, and in our implementation we consider REINFORCE leave-one-out (RLOO)~\citep{salimans2014using,Kool2019Buy4R}
with two samples. Firstly, the exact gradient wrt $w$ of the exact loss is written as 
\begin{equation}   
\label{eq:gradient_w_1}
 - \int_0^1 \frac{1}{t} \E_{q(x_{t}|x_0)}\left[ g(x_t, x_0)\right] \rmd t 
- \int_0^1 \frac{1}{t} \E_{q(x_{t}|x_0)}   \left[ f(x_t, x_0)\nabla_w \log  q(x_{t}|x_0) \right] \rmd t. 
\end{equation} 
where 
$$g(x_t, x_0) 
= \sum_{\indextoken:x_t^{(\indextoken)}=m }  (x_0^{(\indextoken)} - \mu_\theta^{(\indextoken)}(x_t,t) + x_0^{(\indextoken)} (x_0^{(\indextoken)})^\top \log \mu_\theta^{(\indextoken)}(x_t,t)), \ \ \
f(x_t, x_0) 
= w^\top g(x_t, x_0).
$$
Note that $g(x_t, x_0)$ is a vector while $f(x_t,x_0)$ 
is a scalar. 
The left term in 
\eqref{eq:gradient_w_1}
is easy since it just requires 
sampling $t$ and $x_t \sim q(x_t | x_0)$, while the right term is the  REINFORCE term which could have high variance. For this second term we use  RLOO with two samples $x_t^{1}, x_t^2$ and construct the unbiased estimate 
\begin{equation*}   
- \frac{1}{2 t}  \left( \nabla_w \log  q(x_t^1|x_0) - 
\nabla_w \log  q(x_t^2|x_0) \right)  \left[ f(x_t^1, x_0) - f(x_t^2, x_0) \right]. 
\end{equation*} 
Thus, the overall unbiased gradient for $w$ we use is 
\begin{equation*} 
- \frac{1}{2 t} \left\{ g(x_t^1, x_0) + g(x_t^2, x_0)  + \left( \nabla_w \log  q(x_t^1|x_0) - 
\nabla_w \log  q(x_t^2|x_0) \right)  \left[ f(x_t^1, x_0) - f(x_t^2, x_0) \right] \right\}. 
\end{equation*} 

\section{Experimental Details} \label{app:exp-details}

In all experiments, the model is trained with a continuous-time loss while samples are drawn from the discrete-time reverse model of 1000 timesteps unless otherwise noted.
We used an exponential moving average factor 0.9999 for all evaluation including sample generation.

\subsection{text8}
\label{app:exp-text8}

We followed the standard dataset split as in \citet{austin2021structured,lou2023discrete} and trained our models on text chunks of length 256 for 1 million steps with batch size 512. 
All models in the table used a standard 12-layer transformer architecture unless otherwise noted.
Our transformer has also the same number of heads (12) and hidden dimension (784) as in \citet{austin2021structured,lou2023discrete}.

We used the continuous-time ELBO and drew one sample of $t$ for each data to estimate the integral.
To reduce the variance of training, we used the same antithetic sampling trick described in \citet{kingma2021variational} for continuous diffusion models.
We used the linear masking schedule $\alpha_t = 1 - t$ and added a small shift $\epsilon = 10^{-4}$ when $t$ is close to $0$ and $1$ to ensure numerical stability. 
The shifted schedule is $\alpha_t = (1 -2\epsilon) (1 - t)  + \epsilon$. 
The shift leads to a support mismatch between $q(x_1|x_0)$ and the prior $p(x_1)$, leading to an undefined  KL divergence term. 
We explain in \cref{app:undefined-kl} how to modify the prior distribution to allow small uniform probabilities in non-mask states to mitigate this problem.
The shift leads to a non-zero reconstruction term and  KL divergence term for the prior distribution but both are of negligible scale so we can safely ignore them when reporting the ELBO. 

We used a cosine learning rate schedule with a linear warm up of 2000 steps.
We applied channel-wise dropout of rate $0.05$ and used AdamW optimizer with learning rate 0.0003 and a weight decay factor of 0.03. 
Our model is trained on 16 TPU-v5 lite for less than a day.

\subsection{OpenWebText}

We kept 2\% of the original training set for validation. 
Our small and medium transformer model have the same number of layers, heads, and hidden dimensions as in \citet{lou2023discrete} and our tokenizer was also kept the same with a vocabulary size of around 50k. 
The training objective, masking schedule and other architectural choices were kept the same with the text8 experiment. 
We kept the training hyperparameters the same as text8 experiment except that we reduced the dropout rate to 0.02.

\subsection{FineWeb-Edu}
We kept the same training setup as the OpenWebText experiments. Our transformer models have the same number of layers, heads, and hidden dimensions as those of GPT-2 models. We use the same GPT-2 tokenizer.

For Hellaswag evaluation, we concatenate question with each answer option, tokenize the concatenated string, pad to the length of 1024. The padded token sequence gets fed to our MD4 model's loss function for likelihood evaluation. We average 32 Monte Carlo samples to reduce variance. The answer with the highest likelihood estimate is the model's prediction.

\subsection{Images}
\label{app:images}

We used the same linear masking schedule as in previous experiments in all likelihood results. 
We used the same U-Net plus self-attention architectures from the continuous diffusion model described in \citet{kingma2021variational} for CIFAR-10, except that we did not use Fourier feature inputs and added an additional input embedding layer with embedding size the same as the hidden dimension of the model.
For ImageNet $64\times 64$, we reduced the number of residual blocks (in one side of the U-Net structure) from 64 to 48 and added a 12-layer diffusion transformer~\citep{peebles2023scalable} with 768 hidden dimension and 12 heads in the middle. 

For both datasets we used AdamW optimizer and trained for 2M iterations. 
We used learning rate 0.0004, batch size 256, weight decay factor 0.01 for CIFAR-10 and learning rate 0.0002, batch size 512, weight decay factor 0.03 for ImageNet 64$\times$64. 
The learning rate follows a cosine annealing after 100 warm up steps. 
Our CIFAR-10 model is trained on 32 TPU-v5 lite for 24 hours.
Our ImageNet-$64\times64$ model is trained on 256 TPU-v5 lite for 3.5 days.

As explained in \Cref{sec:sampling}, we have observed that the cosine schedule leads to better sample quality so we used it to train a cheaper model for sample visualization. 
This model differs from the one that achieves best likelihood in that we used 8 residual blocks (in one side of the UNet structure) and a 20-layer diffusion transformer in the middle. 
All other configurations are kept the same.

\section{Additional Results} 
\label{app:results}

\subsection{Sample quality evaluation by GPT-2}
 
We use the GPT-2 large model to evaluate the perplexity of samples generated by our model, following \citet{lou2023discrete}.
Results are shown in \Cref{fig:generative-ppl}. 

\begin{figure}[h]
    \centering
    \includegraphics[width=0.7\textwidth]{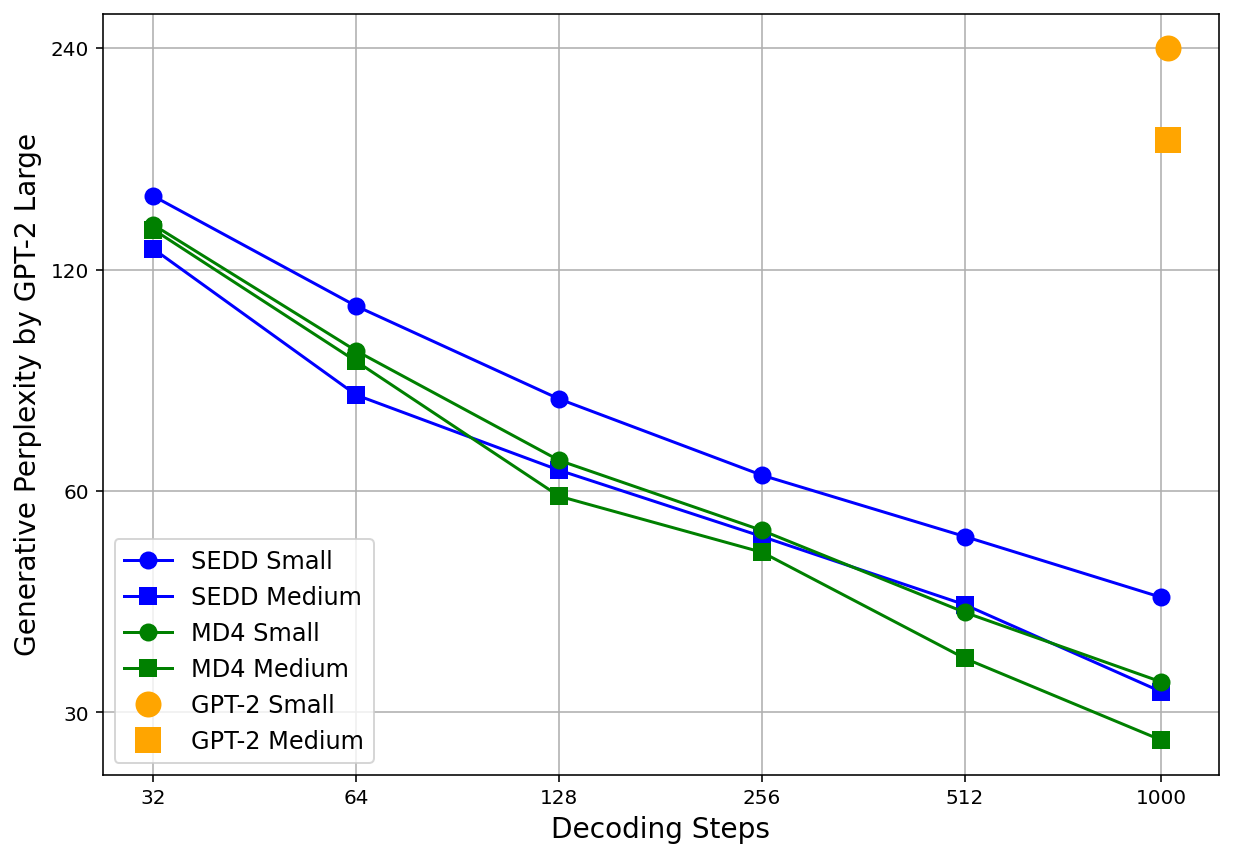}
    \caption{Generative perplexity evaluated by GPT-2 Large following \citet{lou2023discrete}. We compare MD4 against the GPT-2 checkpoint (autoregressive baseline) and SEDD (the previous best discrete diffusion model on this task) in generating 1024-token text sequences. We investigate the effects of two orthogonal factors on sample quality: model size and decoding steps. The numbers for GPT-2 and SEDD are from \citet{lou2023discrete}.} 
    \label{fig:generative-ppl}
\end{figure}

\subsection{Perplexity on OpenWebText validation set}

\Cref{tab:owt-eval-ppl} reports the final perplexity number achieved on OpenWebText validation set, corresponding to  \Cref{fig:owt-eval}. 

\begin{table}[h]
    \centering
    \footnotesize
    \caption{Perplexity on OpenWebText validation set.}
    \label{tab:owt-eval-ppl}
    \vskip 0.1in
    \begin{tabular}{llr}
    Size & Method & Perplexity ($\downarrow$) \\
    \midrule 
    Small  
     &  Gaussian Diffusion & $\le$ 27.28 \\
     &  SEDD Absorb (reimpl.) & $\le$ 24.10 \\
     &  MD4 (Ours) & $\le$ 22.13 \\
     &  GenMD4 (Ours) & $\le$ \bf 21.80 \\
     \midrule
    Medium 
     & MD4 (Ours) & $\leq$ \bf 16.64 \\
    \bottomrule
    \end{tabular}
\end{table}

\subsection{FID evaluation of MD4 trained on ImageNet 64$\times$64}

We provide the FID numbers corresponding to \Cref{fig:imagenet-fid} in \Cref{tab:imagenet-fid}.

\begin{table}[h]
\caption{FID of 50k samples generated by MD4 trained on ImageNet 64$\times$ 64, corresponding to \Cref{fig:imagenet-fid}. Top three rows show results from an unconditional model, while the bottom row is from a model conditioned on class labels. Uniform discretization grid is used in \Cref{alg:sampling} unless otherwise noted. }
\label{tab:imagenet-fid}
\vskip 0.1in
\footnotesize
\centering
\begin{tabular}{lrrrr}
\multirow{2}{*}{Method} &  \multicolumn{4}{c}{Timesteps $T$} \\
& \multicolumn{1}{c}{64} & \multicolumn{1}{c}{128} & \multicolumn{1}{c}{256} & \multicolumn{1}{c}{512} \\
\midrule %
Linear $\alpha_t$ & 193.81 & 128.18 & 72.94 & 50.21\\ 
Linear $\alpha_t$, cosine grid & \bf 42.07 & 25.16 & 18.31 & \bf 18.22\\  
Cosine $\alpha_t$ & 47.46 & \bf 23.84 & \bf 17.8 & 18.74\\ 
\midrule
Cosine $\alpha_t$, class conditional & \bf 30.75 & \bf 11.39 & \bf 7.13 & \bf 7.8 \\    
\bottomrule %
\end{tabular}
\end{table}

\subsection{Additional unconditional generation from MD4 trained on ImageNet 64$\times$64} 
We provide more unconditional generation results from our pixel-level modeling experiments on ImageNet 64$\times$64 in \Cref{fig:imagenet64-app}.

\begin{figure}[h]
    \centering
    \includegraphics[width=0.9\textwidth]{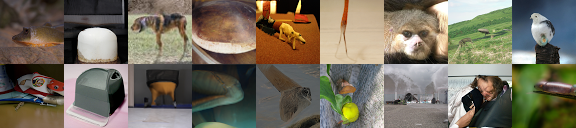}
    \vskip 0.1in
    \includegraphics[width=0.9\textwidth]{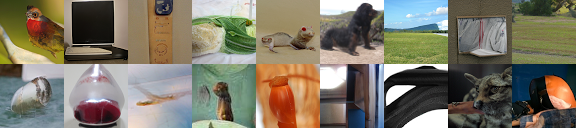}
    \vskip 0.1in
    \includegraphics[width=0.9\textwidth]{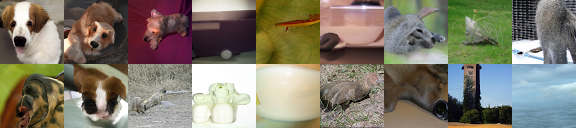}
    \vskip 0.1in
    \includegraphics[width=0.9\textwidth]{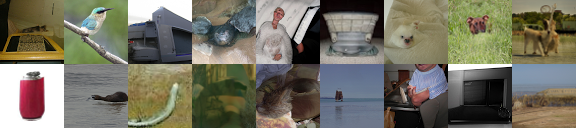}
    \vskip 0.1in
    \includegraphics[width=0.9\textwidth]{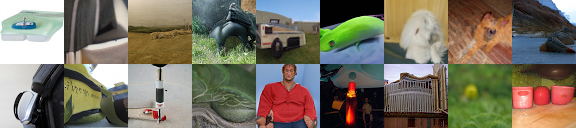}
    \vskip 0.1in
    \includegraphics[width=0.9\textwidth]{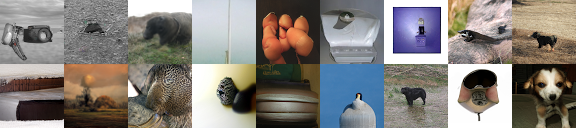}
    \caption{More unconditional samples from MD4 trained on ImageNet 64$\times$64. }
    \label{fig:imagenet64-app}
\end{figure}

\subsection{Additional unconditional generation from MD4 trained on OpenWebText}
\label{app:owt-uncond}

Below we include two unconditioned text samples generated by our MD4 Medium model trained on OpenWebText.

\subsubsection{MD4-M unconditional sample 1: 1024 tokens}
\begin{verbatim}
like, I don't have to be alive? Sometimes there are things that are too real
and you're really supposed to experience them. So that's a good feeling. 
That is the scary thing. Not actually, being able to experience things, being
able to do these things, when you're doing them, which, for most people 
having to wake in a dream is something that seems the most significant, and then
you think about it the next day. It's like the hope of the future, 
and you wake up right now thinking about it. What happens is,, then you
have to stop and think about it and then all of a sudden, somebody always
says, "You're dreaming."

And sometimes I wonder if this is a good time to teach your gut instincts to 
your actors when you're doing a show like this. Because even on this particular 
show, it feels like everyone's been through this all the time before, if even
a few years ago. I mean, if you're doing a show together, at least not on 
continuous development, you you're a vet. I mean, you should really be along.
If you're not sure, well --

VS: I'm working on that one.

Did any of you guys feel that an instinct could work? I thought, "Well, because
you didn't do 'Deadwood' you should stop doing this." But when I read the story
for the first time, I thought, "I think this is going to work." What I can't 
picture is a way to hold this apart.

VS: That's me. It's what we have to do. So do we. When we wrote the first episode,
we wrote a script that we felt like me and myself would want to see. I knew that I
wanted to be able to be in something -- and I wanted to be able to take refuge in
something that was real, that you could see and just really step out of yourself. 
And then I saw it. Then, you get rehearsing it and doing it. And then I actually
started shooting. I think I knew I didn't think it was going to be good. But,
I know it was good. And now people are talked about because it's not good enough.

Growing up, you say that you just completely hated the show, "Lost." Isn't that
what you wish for at the end of the day?

VS: I don't like the concept.

And so there's a lot that you don't know about that, so I think for me to have had
these ideas, if you didn't understand even that it was coming out of this world 
that doesn't exist, we might never get together.

It's so weird. This happened to happen at the same time?

VS: Yes. It happened to happen at basically the same time.

Nobody's even had a show or had a movie/come out of the movie, but ...

VS: If I'm going to pretend I'm definitely not you and have to live through that
stuff, I don't think I'm going to swallow that. I didn't expect it to do quite
that long.

There are always things now that happen with 'Deadwood' where you don't know where
it's going to end up next time, but I think there are occasions now where we have
to keep the fight, even if 'Lost' was pretty consistent in the mindset and the form.

VS: I'm glad that we did fight the odds, because we should have understood that
there was a direct link. But there was almost a sense of not that we had showed up
on the same day, we know we work in the same pieces, but a lot of stuff we don't
know about. Some of it, we need to deal with. We also just have to accept the 
language, and there are a lot of things where we take from them and we do this 
what they did  because we want to
\end{verbatim}

\subsubsection{MD4-M unconditional sample 2: 1024 tokens}
\begin{verbatim}
the groups let recreational vehicles use the three roads that will stay open in 
the meantime of fighting off the permit. "The purpose of the permit is to make 
sure that we work with  the NPS and made roadways and rest areas. We're not just
scaring guys kind of messing around." Community plans to build an urban bike 
facility marched forward at the ongoing staff meeting of the King County
Commission.

Trail will be finished just south of the Greenview 5.

Instead of continuing with a pedestrian and bike trail to the MBTA's campus, these
two trails could bridle the areas from Market to 14 and carry communities closer.

"This project will provide a car-free path to King County," said Andrew Weed. It's
been put the brakes on in the past several months, but there are those residents 
still skeptical.

"I've addressed some of the community concerns that've been raised. They've 
expressed some of their concerns. I don’t think it's terribly reasonable from a 
transportation standpoint."

The trail had been set up to meet on for more than a year when the council 
approved funding for a different proposal.

Mayor Muriel Bowser said after meetings with Commissioner Bushell on Thursday that
the new plan will be on board in December.

"There’s enough of a finish for this project to roll out on time, and we’re going
to get it done," Bowser said.

For the public, the campaign appears over.

“There was one meeting that I feel like I lost at last night's meeting," said 
Shelley Potts, a local resident.

Local resident Joel Grimy, who lives on Uman Road, met residents there as well.

And in other groups that rode through Mayor assistant Stacey Land and even her son
held fliers saying to look for light sign, and also met with Bowser’s son, Deion 
Bowser, about a future plan to also have a dog park on the transit corridor.

Advocates at Brickley’s event, many one waited at least 11 minutes in during the
start of the public meeting, said they expect at least another month from the 
Board of Commissioners, even after a public hearing on Nov. 13.

"We've been trying to be a talkative board where we are meeting in advance, being
respectful of folks," Bowser said.

He considered that the proposal for the section of trail between the Greenview 5 
and 3 “has to move on a schedule. We have other historic preservation projects 
that would take over that.”

But Chad Routledge, a local advocate of the project, spoke out against the mayor’s 
plan.

“The mayor has sent a new meeting to the public using the same route that resulted 
from the loud criticism and onslaught of complaints from the community committee 
back during the public hearing,” Routledge said.

The BDC doesn’t have a particular plan-turns around for the end of the planned 
path, and says “nothing practical can happen right now.” But, she said the agency 
still "looking to make investments in facilities along the route."

And still there is another part of the trail that might be just as much a wish for 
the dogs, as cars: the district wants to go west foot a couple blocks south, to 
make the trail safer for dogs.

“I feel that the accessibility of the trail is pretty important. I think the 
education of the trail, and the uses along different routes are very important 
pieces of a balanced outcome,” said Bushell.

Trams coming off Route 1 
\end{verbatim}

\subsection{Conditional generation from MD4 trained on OpenWebText}

We share conditionally generated text samples by MD4 Medium in \Cref{fig:conditional-sample-generations} and observe that slow unmasking near $t=1$,  enabled by the cosine schedule, tends to help produce more consist and meaningful samples than uniform unmasking counterpart.
\begin{figure}[h]
    \centering
    \includegraphics[trim={0 3cm 0 3cm},clip,width=\textwidth]{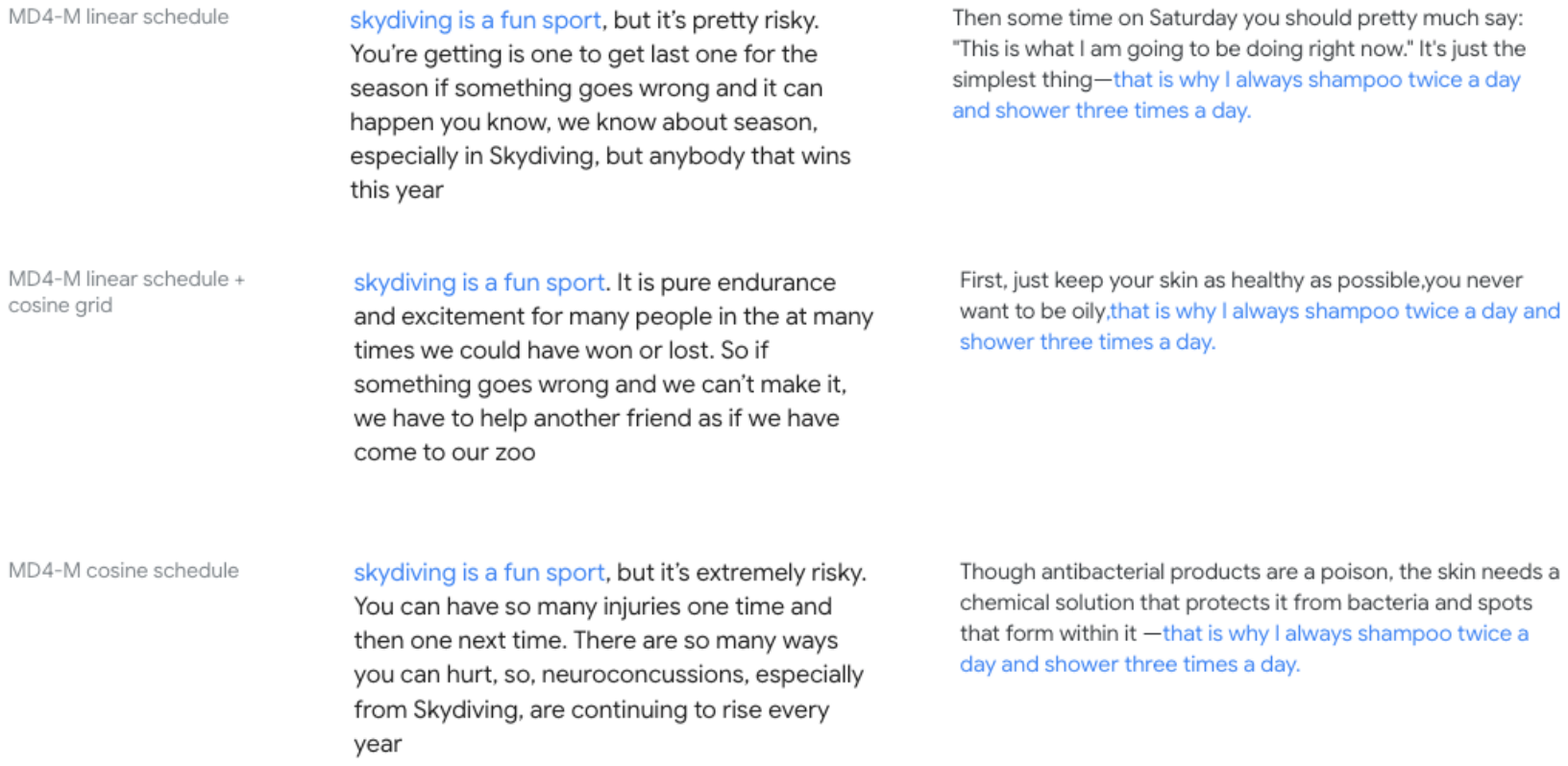}
    \caption{Conditionally generated text samples from MD4-M. Top: MD4-M trained with the linear schedule, sampled with a uniform grid; Middle: MD4-M trained with the linear schedule, sampled with the cosine grid; Bottom: MD4-M trained with the cosine schedule, sampled with a uniform grid. Context text shown in blue, model-generated text in black.}
    \label{fig:conditional-sample-generations}
\end{figure}

\subsection{Effect of discretization on zero-shot perplexity}

We carried out ablation study on the effect of discretization on zero-shot perplexity. 
Results are included in \Cref{tab:continuous_vs_discrete}. 
Note that this is an inference ablation with the same trained model (MD4-S trained with the continuou-time objective).

\begin{table}[h]
\caption{Effect of discretization on zero-shot perplexity.}
\label{tab:continuous_vs_discrete}
\vskip 0.1in
\footnotesize
\centering
\begin{tabular}{llrrrrr}
Size & Timesteps %
& LAMBADA & WikiText2 & PTB & WikiText103 & IBW \\
\midrule %
Small & T = 100 %
& $\le$ 49.8 & $\le$ 36.1 & $\le$ 105.2 & $\le$ 36.1 & $\le$ 70.3\\ 
    & T = 1000 %
        & $\le$ 48.5 & $\le$ 35.0 & $\le$ 102.5 & $\le$ 35.0 & $\le$ 68.4\\  
    & T = 10000 %
      & $\le$ 48.4 & $\le$ 34.9 & $\le$ 102.4 & $\le$ 34.9 & $\le$ 68.2\\ 
    & T = $\infty$ (continuous) %
      & $\le$ 48.4 & $\le$ 34.9 & $\le$ 102.3 & $\le$ 35.9 & $\le$ 68.1\\    
\bottomrule %
\end{tabular}
\end{table}

\end{document}